\documentclass[nonblindrev]{style_ms}
\usepackage[utf8]{inputenc}
\usepackage[OT1]{fontenc}

\usepackage[protrusion=true,expansion=true,final,babel]{microtype}
\usepackage{cancel}

\usepackage{bbm}
\usepackage{booktabs}
\usepackage{amsmath}
\usepackage{thmtools,thm-restate}

\usepackage{algorithm}
\usepackage{algpseudocode}

\usepackage{scalerel,stackengine}
\usepackage{lipsum}

\newcommand\extrafootertext[1]{%
    \bgroup
    \renewcommand\thefootnote{}%
    \renewcommand\thempfootnote{}%
    \footnotetext[0]{#1}%
    \egroup
}

\usepackage[colorlinks,linkcolor=red,anchorcolor=blue,citecolor=blue]{hyperref}

\definecolor{purple}{rgb}{0.59, 0.44, 0.84}
\stackMath
\newcommand\reallywidehat[1]{%
	\savestack{\tmpbox}{\stretchto{%
			\scaleto{%
				\scalerel*[\widthof{\ensuremath{#1}}]{\kern-.6pt\bigwedge\kern-.6pt}%
				{\rule[-\textheight/2]{1ex}{\textheight}}
			}{\textheight}%
		}{0.5ex}}%
	\stackon[1pt]{#1}{\tmpbox}%
}
\newcommand{\longoverbrace}[2]{\overbrace{#1}^{\text{\hbox to 0cm{\hss #2 \hss}}}}
\newcommand{\uunderbrace}[2]{\underbrace{#1}_{\text{\hbox to 0cm{\hss #2 \hss}}}}
\newcommand{\demnot}{D}
\newcommand{\vltnot}{v}

\newcommand{\covariatenot}{x}
\newcommand{\fillnot}{f}
\newcommand{\pricenot}{p}
\newcommand{\costnot}{c}
\newcommand{\invnot}{I}

\newcommand{\cA}{\mathcal{A}}

\newcommand{\cD}{\mathcal{D}}

\newcommand{\beforeinv}{I^i_{t_{-}}}
\newcommand{\paramnot}{\theta}
\newcommand{\bigparamnot}{\Theta}
\newcommand{\barepolicynot}{{\pi}}
\newcommand{\policynot}{\barepolicynot^i_{\paramnot, t}}
\newcommand{\actionspace}{\mathbb{A}}

\newcommand{\setpolicynot}{{\bf\Pi}}

\newcommand{\measurenot}{\mathbbm{P}}
\newcommand{\salvagenot}{b}
\newcommand{\btheta}{\mathbf{\paramnot}}
\newcommand{\productsumnot}{\frac{1}{|\cA|}\sum_{i\in \cA}}

\newcommand{\rewardrappendix}{R (H_t, \paramnot)}

\newcommand{\rewardexp}{\mathbb{E}^{\measurenot}}

\newcommand{\tv}[1]{\|#1\|_{\text{TV}}} 

\TheoremsNumberedThrough     
\ECRepeatTheorems

\EquationsNumberedThrough    

\MANUSCRIPTNO{}

\begin{document}

\RUNAUTHOR{Madeka et al.}
\RUNTITLE{Deep Inventory Management}

\TITLE{Deep Inventory Management}

\ARTICLEAUTHORS{%
	\AUTHOR{Dhruv Madeka}
	\AFF{Amazon, \EMAIL{maded@amazon.com}} 
	\AUTHOR{Kari Torkkola}
	\AFF{Amazon, \EMAIL{karito@amazon.com}}
	\AUTHOR{Carson Eisenach}
	\AFF{Amazon, \EMAIL{ceisen@amazon.com}}
    \AUTHOR{Anna Luo}
	\AFF{Pinterest\footnote{Work done while at Amazon.}, \EMAIL{annaluo676@gmail.com}}
  \AUTHOR{Dean P. Foster}
	\AFF{Amazon, \EMAIL{foster@amazon.com}}
  \AUTHOR{Sham M. Kakade}
	\AFF{Amazon, Harvard University, \EMAIL{shamisme@amazon.com}}
	\extrafootertext{*\textbf{Acknowledgments: } The authors would
like to thank Romain Menegaux, Robert Stine, Alvaro Maggiar, Salal Humair, Ping Xu,
Vafa Khoshaein, Yash Kanoria, and numerous others from the Supply Chain Optimization Technologies group for their invaluable feedback.}

} 



\ABSTRACT{
  This work provides a Deep Reinforcement Learning approach to solving a periodic
  review inventory control system with stochastic vendor lead times, lost sales,
  correlated demand, and price matching. While this dynamic program has
  historically been considered intractable, our results show that several policy learning
  approaches are competitive with or outperform classical methods.
  In order to train these algorithms, we develop novel techniques to convert
  historical data into a simulator. On the theoretical side, we present learnability
  results on a subclass of inventory control problems, where we provide a
  provable reduction of the reinforcement learning problem to that of supervised
  learning. On the algorithmic side, we present a model-based reinforcement
  learning procedure (Direct Backprop) to solve the periodic review
  inventory control problem by constructing a {\it differentiable simulator}.
  Under a variety of metrics Direct Backprop outperforms model-free RL and
  newsvendor baselines, in both simulations and real-world deployments.
}

\KEYWORDS{reinforcement learning, inventory control, differentiable simulation}

\maketitle

\section{Introduction}

A periodic review inventory control system determines the optimal inventory level that should be held for different products by attempting to balance the cost of meeting customer demand with the cost of holding too much inventory. Inventory level is reviewed periodically and adjustments are made by procuring more inventory or removing existing inventory through various means. The inventory management problem can be abstracted as a Markov Decision Process (MDP) \cite{porteus2002foundations} - though a number of complexities make it difficult or impossible to solve using traditional dynamic program (DP) methods.

A major complexity is that demand is not a constant value or a deterministic function, but a random variable with unknown dynamics that exhibit seasonalities, temporal correlations, trends and spikes. Demand that is not met by the inventory in the warehouse or store (called \textit{on-hand inventory}) is lost since customers tend to go to competitors. This results in a non-linearity in the state evolution dynamics when demand is lost, and a censoring\footnote{i.e., a lack of observability.} of the historical data.

Another complexity that arises is that an order placed by a retailer to a vendor incurs a time lag between the actual placement of the order and the arrival of the items in the warehouse (usually called the \textit{Vendor Lead Time or VLT}), another random variable. For a modern retailer, price matching is also another concern \cite{janssen2013price} - as exogenous changes in prices by competitors force the price to have properties outside of being just another decision variable.

In this work, we present the first Deep Reinforcement Learning (DRL) based periodic review inventory system that is able to handle many of the challenges that make the DP solution intractable. Our model is able to handle lost sales, correlated demand, stochastic vendor lead-times and exogenous price matching. We also present new techniques in correcting historical data to mitigate the issues arising from the fact that this data is ``off-policy''. These techniques allow us use historical data directly as a simulator, as opposed to building models of the various state variables. Finally, we observe that the ``state-evolution'' of our target application is both known, {\it and differentiable}. To that end, we propose a model-based reinforcement learning (RL) approach (DirectBackprop) where we directly optimize the policy objective by constructing a differentiable simulator. This approach can be used to solve a class of decision making problems with continuous reward and transition structure (of which inventory control is a special case). To motivate our work, we present a collection of results in Section \ref{sec:learnability} that illustrate why the inventory management problem is efficiently learnable from historical data. Finally we show in Section \ref{sec:results}, that for the inventory management problem, DirectBackprop outperforms model free RL methods which in turn outperform newsvendor baselines.

\paragraph{Our contributions.}

Our main contribution is to provide a Deep Reinforcement Learning approach to
solving a periodic review inventory control system with stochastic vendor lead
times, lost sales, correlated demand, and price matching. Specifically:
\begin{enumerate}
\item By modifying historical data we are able to model the inventory management
    problem as a largely exogenous Interactive Decision Process.
\item  We present a novel algorithm `DirectBackprop' that utilizes the
differentiable nature of the problem to achieve state-of-the-art
performance.
\item We present a collection of learnability results that motivate our approach
and provide understanding of its empirical efficacy.
\item  Finally, we present empirical results from the first
large-scale deployed end to end Deep RL system that considers a realistic
discrete-time periodic review inventory management system with
non-independent demand, price matching, random lead times and lost
sales at each decision epoch.
\end{enumerate}

\paragraph{Organization} The first part of the paper (Section \ref{sec:formulation})
sets up the mathematical formulation and framework for the Interactive Decision
 Process (IDP) we are solving. Section \ref{sec:acdemand} describes the techniques
we use to convert our historical data into a reliable simulator. Sections \ref{sec:rl_policies}
and \ref{sec:directbp} cast the periodic review inventory system as a
differential control problem and describe how this differentiable simulator can
be used to train neural policies.

Section \ref{sec:learnability} presents a collection of theoretical results that illustrate
how our framework can be used to backtest \textbf{any} policy, not just those
based on Reinforcement Learning. Theorem \ref{theorem:censoredlearnability} describes
how these results apply even when some of the historical data is unobserved.
In Section \ref{sec:results}, we present a collection of experimental results in
both simulated and real-world settings.

\section{Related Work} \label{sec:ref}

Inventory models \cite{maggiar2017joint,karlin1958inventory,scarf1959optimality} often assume that any sales that are missed due to a lack of inventory (usually called Out of Stock or OOS) are backlogged and filled in the next period. It is estimated through the analysis of large scale surveys that only 15\% - 23\% of customers are willing to delay a purchase when confronted by an out-of-stock item \cite{gruen2002retail,verhoef2006out}. This indicates that a lost sales assumption is more realistic in competitive environments than a full-backlogging assumption. See \cite{bijvank2011lost} for a comprehensive review of the lost sales inventory literature.

Karlin et al. \cite{karlin1958inventory} study an inventory system with a continuous demand density, lost sales and a constant lead time of 1. They show that a base stock policy (i.e. a policy that orders up to a target inventory level)\footnote{For a formal definition of base stock policies, see \cite{feng2006optimality} and references within.} is suboptimal in this case even with linear cost structures, pointing out that the way ``a lag in delivery influences the decision process is analogous to that of a convex ordering cost.'' This indicates that a newsvendor model with a base stock policy is already suboptimal for a modern retailer's inventory system. Under certain conditions, such as a fixed lead time with a large lost demand penalty cost, it can be shown that the base stock policy for a backordered model can be asymptotically optimal for a lost sales model \cite{huh2009asymptotic}. However, as in our problem, with longer lead times and a fixed penalty, a constant order policy may outperform a base stock policy \cite{reiman2004new}.

Our problem further deviates from the literature in that our lead times are modeled as a random variable. Inventory systems with lost sales and stochastic lead times are scarcely studied in the literature. Zipkin \cite{zipkin2008old} showed that even for a simple setting of lead times that are constant (but larger than one time unit), stochastic demands, and lost sales - base stock policies tend to be numerically worse than myopic or constant order policies as the lead time increases. Kaplan \cite{kaplan1970dynamic} demonstrates how, when the maximum delay period is constrained, the state space for a random lead time model can be reduced to a single dimension. Nahmias et al. \cite{nahmias1979simple} formulate myopic policies for an inventory system where excess sales are lost and the lead time is random. Janakiraman et al. \cite{janakiraman2004lost} develop convexity results for an inventory management system with lost sales, non-crossing lead times and order up to policies.

Demand censoring has been studied in the context of the spiral down effect in
revenue management \cite{cooper2006models} as well as for the censored
newsvendor problem (see \cite{besbes2013implications} and citations within). We
have not seen a historical data correction similar to the one in Section
\ref{sec:acdemand} being applied to the policy learning problem framed in
Section \ref{sec:rl_policies}.

For large retailers in competitive markets, price matching causes abrupt and random changes in the price of a product during a decision epoch. This causes the price to behave like an exogenous stochastic process, as opposed to an endogenous decision variable. There is no literature we could find that addresses the problem of a periodic review inventory system with price matching effects.

Reinforcement learning has been applied in many sequential decision making problems where learning from a large number of simulations is possible, such as games and simulated physics models of multi-joint systems. While the use of deep learning for forecasting demand has developed recently \cite{mqcnn, madeka2018sample, wen2019deep, eisenach2020, yang2022mqretnn}, the usage of reinforcement learning to directly produce decisions in Inventory Management has been limited. Giannoccaro et al. \cite{giannoccaro2002inventory} consider a periodic review system with fixed prices, costs, and full backlogging of demand and show that the SMART RL Algorithm \cite{rlsmart} can outperform baseline integrated policies. Oroojlooyjadid et al. \cite{oroojlooyjadid2016applying} integrate the forecasting and optimization steps in a newsvendor problem by using a Deep Neural Network (DNN) to directly predict a quantity that minimizes the newsvendor loss function. Balaji et al. \cite{balaji2019orl} study the multi-period newsvendor problem with uniformly distributed price and cost, constant VLT and stationary, Poisson distributed demand and show that Proximal Policy Optimization \cite{PPO} can eventually beat standard newsvendor benchmarks.

Gijsbrechts et al. \cite{gisjbrechts2021can} consider lost sales, dual sourcing and multi-echelon problems and show that modern Deep Reinforcement Learning algorithms such as A3C \cite{a3c} are competitive with state-of-the-art heuristics. Qi et al. \cite{qi2020practical} train a neural network to predict the outputs of an (ex-post) ``oracle'' policy. While this means that their approach does not follow a Predict-then-Optimize (PTO) framework, the oracle policy does not allow them to handle how the variability of the exogenous variables (such as demand, price etc.) might influence the optimal \textit{ex-ante} policy.\footnote{By taking gradients against the true reward function, our approach is able to handle this variability}

Differentiable simulators have been studied
\cite{suh2022differentiable} and applied to problems varying from
physics \cite{hu2019chainqueen} and protein folding
\cite{ingraham2018learning}. In the context of inventory management, they have
been applied \cite{glasserman1995sensitivity} to studying the sensitivity of
inventory costs to optimal parameters for base stock levels. They have not been
studied in the context of directly learning (neural) policies through the gradients from
the simulator.

Exogenous Markov Decision Processes and their
learnability have been studied \cite{sinclair2022hindsight}, our work allows the
policy to be dependent on the entire trajectory of the Exogenous random
variables. \nocite{efroni2022sample} \nocite{efroni2022sparsity}

\section{Problem Formulation}
\label{sec:formulation}

\subsection{Mathematical Notation}

Denote by $\mathbb{R}$ and $\mathbb{N}$ the set of real and natural
 numbers respectively. Let $(\cdot)^+$ refer to the classical
 positive part operator i.e. $(\cdot)^{+} = \max(\cdot, 0)$. The
 inventory management problem seeks to find the optimal inventory
 level for each product $i$ in the set of retailer's products, which
 we denote by $\mathcal{A}$. We assume all random variables are
 defined on a canonical probability space
 $(\Omega, \mathcal{F}, \measurenot)$. Let
 $\paramnot \in \bigparamnot$ denotes some parameter set.  We use
 $\mathbb{E}^{\measurenot}$ to denote an expectation operator of a
 random variable with respect to some probability measure
 $\measurenot$. Let $\tv{\mathbb{Q}_1, \mathbb{Q}_2}$ denote the total variation
 distance between two probability measures $\mathbb{Q}_1$ and $\mathbb{Q}_2$.

\subsection{Construction of the Interactive Decision Process}

We will construct our Interactive Decision Process (IDP) in two steps:
First we will describe the driving stochastic ``noise'' processes
which govern our problem.  These are things like demand, price
changes, cost and other things that are outside of our control.  The
assumption is that nothing in this set is influenced by our actions.
Second, we will describe our decision process which can depend on all
the previous information contained in the above processes.

\paragraph{External processes:}  In order to add clarity to our construction, we begin with a single
 product $i$.
We denote by $\demnot^i_t \in [0, \infty)$ the random variable
 for demand at $t$. We adopt the convention that demand arrives at the end
of a time period.  Our main source of time lag comes from vendor lead times, which refer
 to the time between placing an order and its subsequent
 delivery. Denote by $\vltnot^i_t \in [0, \infty)$ the random variable
 for vendor lead time at the decision epoch $t$ for product
 $i$. Orders are allowed to cross, that is, an order placed at a later
 decision epoch may arrive before an order placed at an earlier
 decision epoch. Formally, the sequence of random variables
 $t + v^i_t$ is \textbf{not necessarily} increasing in $t$.
Similarly, denote  $\pricenot^i_t \in [0, \infty)$ and $\costnot^i_t \in [0, \infty)$ the random variables associated with the sales revenue (price), and
purchase cost of a product $i$ at time $t$.

We will bundle these together into a state
vector
\begin{displaymath}
s^i_t = (\demnot^i_t, \vltnot^i_t, \pricenot^i_t, \costnot^i_t) \in {\cal S}.
\end{displaymath}
None of these variables will be under our control and their evolution
will not be influenced by any action we take.

While we focus on the more general setting where there is a full joint distribution
over these external processes, our learnability results in Section \ref{sec:learnability}
will focus on the special case where the distributions over the processes are
independent among different products.

\paragraph{Control processes:} Let $\actionspace := \mathbb{R}^{|\mathcal{A}|}$
denote the set of all possible actions that the decision maker can take for each
possible product. For product $i$ at time $t$ the action taken will be denoted by
$a^i_t \in \mathbb{R}$, which maybe to place an order with a vendor ($a^i_t > 0$) or
to return some units to the vendor ($a^i_t < 0$), or to simply do
nothing ($a^i_t = 0$).  The action for product $i$ at time $t$ is allowed to depend on all the information up to time
$t$.  When items arrive, they
will be added to the on-hand inventory (which we denote by $\beforeinv$).
Since the decision will need to deal with changing demand, it is
important that it be allowed to see the entire history up to time $t$
and not simply the current state $s^i_t$. For simplicity and due to the fact that
deterministic policies can represent the optimal policy, our work will focus
primarily on deterministic actions. %
 Consider the ``history'' vector upto time $t$,
 $H_t := (s^1_1, \ldots, s^1_{t-1}, \ldots, s^{|\mathcal{A}|}_1,\ldots, s^{|\mathcal{A}|}_{t-1})$
We can write our policy as $\policynot: \overset{|{\cal A}|}{\bigotimes}\left(\overset{t-1}{\bigotimes}\mathcal{S}\right) \rightarrow \mathbb{R}$,
and our action $a^i_t$ as:

\begin{equation}
\label{eqn:globala}
a^i_t = \policynot(H_t)
\end{equation}

Examples of states and actions can be seen in Section
\ref{sec:stateaction}.\footnote{Model-Free Reinforcement Learning agents used
in our comparisons require taking random actions in order to explore. We can
rewrite $\policynot$ to include an argument that contains a uniform random
variable in order to capture this behavior.}

Note, in this scenario - there is a separate policy for each product. We characterize
the set of these policies as $\setpolicynot =\{\policynot \mid
\theta \in \Theta, i \in \mathcal{A}, t \in [0, T]\}$. The constrained set
$\setpolicynot$ is part of our algorithmic approach. In our setup, the optimal
policy may not be tractable - which leads us to consider a broad class of neural
policies in Sections \ref{sec:directbp} and \ref{sec:model_free}.

\paragraph{Reward function:} The reward in period $t$ depends on the current
sales (contained in $s^i_t$), the inventory at the start of the period and what
actions are taken: \begin{displaymath} R_t^i := R(H_t,\theta), \end{displaymath}
where positive values of $R_t^i$ reflect income and negative values of $R_t^i$
reflect costs.\footnote{Note, we can reconstruct $\beforeinv$ from $H_t$ and
$\paramnot$.} We assume that all rewards are bounded above by some value
$R^{\text{max}}$ and bounded from below by some value
$R^{\text{min}}$.\footnote{Note that by convention we assume that the order of
events places the action before the state update}We denote the multiplicative
discount factor by $\gamma \in [0, 1]$. The reward function for the inventory
management problem is fully described in Section \ref{sec:objective_function}.

Our IDP becomes
 $\langle {\cal F}, \actionspace, \sum_{i \in \mathcal{A}}R^i,
 \measurenot, \gamma, \setpolicynot \rangle$. This formulation allows us to have as
 context for any product, the features of another product.

\subsection{Objective Function}\label{sec:objective_function}

The reward function at each decision epoch simply measures the cash
 flows, including the revenue from sales at each decision epoch, the
 cost of purchasing, and in-bounding inventory, following textbook
 approaches \cite{porteus2002foundations}. The reward is:
\begin{equation}
        R^i_t := \uunderbrace{p^i_t}{\text{Revenue Terms}} \min(\demnot^i_t, \beforeinv) - \uunderbrace{c^i_t}{\text{Cost Terms}} a^i_t ,
	\label{eqn:approx_reward}
\end{equation}
where we define $I^i_{t_{-}}$ below in (\ref{eqn:transitionper1}).  We
will often make the dependence on the policy explicit by writing $R^i_t(\paramnot)$.
  We
aim to maximize the total discounted reward, expressed as the
following optimization problem:
\begin{align}
	\max_{\paramnot} & ~\rewardexp \biggl[ \sum_{i \in \mathcal{A}} \sum^{\infty}_{t=0} \gamma^t R^i_t(\paramnot) \biggr] \label{eq:proxy}       \\
	\text{subject to: }     &                                                                                                                      \\
	I^i_0                   & = k^i                                                                                                                \\
	a^i_t                   & = \policynot(H_t)                                                                                                         \\
	\beforeinv             & = I^i_{t-1} + \sum_{0 \leq k < t} a^i_k \mathbbm{1}_{\vltnot^i_k = (t-k)}   \label{eqn:transitionper1} \\
	I^i_{t}                 & = \min(\beforeinv - \demnot^i_t, 0). \label{eqn:transitionper2}
\end{align}
Here $k^i$ denotes a non-negative real number that is typically a
preset constant for each product. Possible values for $k^i$ can be
found in Section \ref{sec:initterm}.

We adopt a convention that the supply arrives before the demand for
 any time $t$ in Equation \eqref{eqn:transitionper1}.  This creates a
 state $ I^i_{t_{-}}$ of the inventory at the start of the week that
 is used to fulfill demand.  Furthermore, we implicitly assume that
 every product has a review period of one time step. Note that this
 formulation also implicitly assumes that total costs are a linear
 function of the order size (i.e. there is no bulk discounting in this
 formulation).

Of the past orders placed to vendors $a^i_k$, $k\in(0,\ldots,t)$ those
that have not yet arrived at time $t$ are called in-flight orders. Of
the past orders, the subset that arrives at period $t$ is expressed by
the indicator component $\mathbbm{1}_{\{\vltnot^i_k = (t-k)\}} $ in
the main state update Equation \eqref{eqn:transitionper2}. A more
general problem formulation that incorporates cross-product
constraints can be found in Appendix \ref{app:objective}.

\section{Re-using Historical Data}
\label{sec:censoring}

The goal now is to find a policy $a^i_t = \policynot(H_t)$ such that Equation
\eqref{eq:proxy} is maximized. Besides discrete optimization through a DP, other
approaches include optimal control using methods such as Model Predictive
Control (MPC) \cite{MPC}, and {\em learning} the policy from data, for example,
using Reinforcement Learning (RL). RL typically obtains data by acting in a real
or simulated environment. The key contribution of our work is that rather than
assuming parametric forms and simulating the future, we modify historical data
to allow us to treat each historical trajectory as one path from a simulation.
Since the external processes are not functions of our actions, we are able to
use the historical data to evaluate {\bf any} policy. Our learnability results
in Section \ref{sec:fulllearnability} show how this property allows us to make a
provable reduction of our IDP to Supervised Learning.

For example, historical prices can be treated as exogenous variables that
are outside of the agent's control (this primarily happens through the
mechanism of price matching \cite{janssen2013price}).

\subsection{Censoring of historical data}
\label{subsec:censoring}

Note that since the historical data from any retailer is usually generated by
some policy acting on the world at the time of collection. As a result, the data
needs to be corrected to avoid any censoring that may have occured. For example,
demand may be censored (i.e. not observed) for time periods where there was not
enough inventory in the actual warehouses \cite{Zeni2001censored}. Further
policy dependent data  issues might be that Vendor Lead Times may not be
observed for weeks when no order was placed, or that unit costs depend on the
size of the order placed. Historical data is thus ``on-policy'' only for the
old policy, and ``off-policy'' for any algorithm learning a new policy.

\paragraph{Illustrative Example: Censored Demand} One of the
primary sources of ``off-policy'' bias from historical data comes from the fact
that historical demand is censored for the time periods where there was too
little inventory in the warehouses. Concretely, from Equation
\ref{eqn:approx_reward}, we can define the observed sales for a product $i$ at
time $t$ can be written as: \[ \text{Sales}^i_t := \min \bigl(
\text{Demand}^i_t, \text{Inventory}^i_t \bigr). \]

We need to correct for this in order to prevent the policy from failing to see
sales that might exceed the historical inventory constraints. One approach to
handling this issue is outlined in Section \ref{sec:acdemand}.

We can reframe the construction of our IDP to incorporate the fact that our
historical data may be missing/censored. We defer this construction to Section
\ref{sec:partiallearnability}, which also presents concrete learnability results
for this setting.

    \subsection{Availability Corrected Demand} \label{sec:acdemand}

    We present an approach to ``correcting'' the censoring in the historical
    data by leveraging the fact that an online retailer may see signals of
    demand other than sales even when an item is out of stock, through web page
    glance views. We can classify views as those that came when the item was
    ``in-stock'' or those that came in at all as ``total.'' We construct a
    counterfactual demand signal $\Tilde{\demnot}^{i}_t$, by correcting for
    availability, through the use of glance view counts. In the absence of these
    kinds of related signals, the literature has used standard missing data
    imputation methods to correct for censored demand \cite{Zeni2001censored}.

    Define the conversion rate from {\em glance views} to demand of a product $i$ at
time $t$ as $C^i$. If we are out of stock, we can divide these glance views into
two pieces, those where the customer was made an offer (called {\em in stock
glance views}, say $\overline{\hbox{GV}}^i_t$) and those where no offer was
extended due to being out of stock ({\em out of stock glance views}, say
$\underline{\hbox{GV}}^i_t$).  So $GV^i_t = \overline{\hbox{GV}}^i_t +
\underline{\hbox{GV}}^i_t$.  We can likewise divide up our demand into these
same two pieces so $D^i_t = \overline{D}^i_t + \underline{D}^i_t$, where we can
observe $\overline{D}^i_t$ since we were in-stock and so the customer's purchase
decision is observed (corresponding to the actual realized sales for product $i$
at time $t$). We cannot observe $\underline{D}^i_t$ due to the retailer failing
to make an offer at all.  We assume that $E(\underline{D}^i_t|{\cal F}_{t-1}) =
C^i \underline{\hbox{GV}}^i_t$. No assumption is needed about how the in stock
glance views are related to the in-stock demand since we can observe sales
directly.  So our availability corrected demand (AC Demand) can be estimated by:
    \begin{equation}
        \Tilde{\demnot}^{i}_t = \overline{D}^i_t + C^i \underline{\hbox{GV}}^i_t
        \label{eqn:ac_demand}
    \end{equation}

    \section{Implementation Details}

    The purpose of this section is to describe numerous details of how the
historical data is used as a simulator. This simulator later interacts with
various reinforcement learning algorithms through the OpenAI Gym interface
\cite{brockman2016openai}. Appendix \ref{app:software} describes the software
design of our simulator.

    \subsection{Rollouts}
    Using Equation \eqref{eqn:ac_demand} allows us to estimate a counterfactual
demand signal independent of the agent's order actions allowing us to use
historical data as a simulator to train RL algorithms. We treat each product’s
demand (and separately price/cost) as a sample from some joint stochastic
processes with the appropriate context and features. We choose to have a single
parameterized policy for all products, and use different products rollouts to
generate policy gradient updates.

    \subsection{Initialization and Termination} \label{sec:initterm}

    Starting inventory, $I^i_0$ is left as a ``choice'' variable ($k^i$ in Equation \eqref{eqn:startinv}). We use three different cases to study the behavior of the different agents:
    \begin{itemize}
        \itemsep -0mm
        \item Case 1: $I^i_0 = 0$.  Idealized scenario where	the retailer had no inventory today.
        \item Case 2: $I^i_0 =\text{OnHand}^i$. Initialization to the total historical inventory held by the retailer.
        \item Case 3: $I^i_0 =\text{Policy Inventory}^i$. This is a more stable case (to remove the noise of the real world) of initializing the inventory where another policy left off.
    \end{itemize}
    In practice the rollouts have a finite duration and must be terminated with some value. Following \cite{porteus2002foundations}  we define it as
    \begin{equation}
        V^{\paramnot, M}_t := \rewardexp \biggl[ \sum^{M}_{j
    = t} \gamma^{j-t} R^i_{j}(\paramnot)  \bigg| {\cal F}_t \biggr].
    \end{equation}
    for some large value $M \gg t$. We note here that in a Markovian setting, $\underset{M \rightarrow \infty}{\lim} V^{\paramnot, M}_t$ is well-defined but it may not be in our setting.
    With this value function, the objective  can be expressed as:
    \begin{equation}
     \label{eqn:bellmanversion}
        \max_{\paramnot}  \sum_{i \in \mathcal{A}} \rewardexp \biggl[\sum^{N}_{t=0} \gamma^t R^i_t(\paramnot)  + V^{\paramnot, M}_N \biggr].
    \end{equation}
The terminal value $V^{\paramnot, M}_N$ can be approximated in several ways:
\begin{itemize}
\item $V^{\paramnot, M}_N := \salvagenot \sum_i c^i_N$ ~~Return at
$\salvagenot$ fraction of cost,
\item Retailer shuts down at $t=N$:
\begin{equation}
 \label{eqn:terminal_value}
V^{\paramnot, M}_N := 0,
\end{equation}
\item $V^{\paramnot, M}_N := \text{V}^{\theta', M}_N$, ~~Another policy's value function.
\end{itemize}
    Termination will have an effect on the nature of the learned policy. If inventory is discarded at the end of the rollout, the agent will be cautious in procuring too much extra inventory, whereas selling the inventory back to the vendor at a high price will encourage the opposite behavior.

    \subsection{State and Action Spaces} \label{sec:stateaction}
    \label{sec:state}
    At time $t$ the agent knows $H_t$ which contains
$(s^i_b)_{{i\in \{1,\ldots,|{\cal A}|\},b\in\{1,\ldots,t-1\}}}$ which contains a subset of the features used in
\cite{mqcnn}, containing demand and glance view signals, distances to
public holidays, and product specific features. In addition we include
economics of the product (price, cost etc.), inventory level, past
actions, and in-flight inventory. Lags of time-series are included as
well. Appendix \ref{app:featurization} contains more details about the
features that an agent sees at each point in time.

The single-dimensional continuous action $a^i_t \in \mathbb{R}$ at time $t$ for product $i$ denotes the order quantity directly, rather than a target inventory position.\footnote{Note that for the reported results in Section \ref{sec:results}, the agents were trained without the ability to discard inventory (i.e. $a^i_t \geq 0$).} This prevents an unnecessary dependency on ``order-up'' (or base-stock) policies \cite{feng2006optimality}, which are known to be suboptimal for periodic review inventory systems with stochastic lead times and lost sales (see \cite{janakiraman2004lost} and the citations within).

    The action can be stochastic or deterministic depending on the learning algorithm. To represent a stochastic action, our policy networks output the mean and standard deviation of a Gaussian. A deterministic action is represented by a scalar output.

    Since products can have orders of magnitude differences in their demand
    per period (in particular retail demand has been shown to be very heavy
    tailed \cite{tripuraneni2021assessment}), the policy may have vastly
    different scales of orders for different products, which becomes a
    potential problem as pointed out in \cite{Schulman2016NutsBolts}. We
    use historical demand to provide ``scaling'' factors for the current
    action. Thus the actual order placed when an action is taken by an RL agent is:
    \begin{equation}
        a^i_t = \policynot(H_t) \cdot \text{Scaling Factor}^i_t,
    \end{equation}
    where a number of different values can be used for the scaling factor. The
    scaling factor -- which we use to scale both the input features and the
    policy output -- is
    \[
        \text{Scaling Factor}^i_t := \max\{1, \frac{1}{L}\sum^L_{k=1} D^i_{t-k}\},
    \]
    for some number of trailing time periods $L$. Note that this scaling factor
    is time varying, and thus adapted to the data available at time $t$.

    \section{Policy Optimization}  \label{sec:rl_policies}

    The primary goal of the inventory management problem is maximizing the IDP
    given by Equation \eqref{eq:proxy}. The generic differential control problem
    can be defined as:
    \begin{equation}
            J_T(\paramnot) := \rewardexp \left[\productsumnot \sum^T_{t=0} \gamma^t
R^i_t(\paramnot) \right]
    \end{equation}

    Note that constraints on the action space can be pushed inside the
reward function $R$ so they are omitted from this formulation.

    Ideally we would like to solve the problem:

    \begin{equation}
             \underset{\paramnot \in \bigparamnot}{\text{max}} ~J_T(\paramnot) \\
        \label{eqn:optimal_control}
    \end{equation}

    Typically solving such a maximization problem requires computing derivatives
    with respect to $\paramnot$. So we would ideally like to compute:

    \begin{equation}
            \frac{\partial}{\partial \paramnot} J_T(\paramnot)
        \label{eq:policy_derivative}
    \end{equation}

    It is impossible to compute this derivative exactly, so instead we
    will use a stochastic
    gradient, which is an estimate $\widehat{\frac{\partial J}{\partial \paramnot}}$ such
    that $\mathbb{E}\left[\widehat{\frac{\partial J}{\partial \paramnot}}\right] =
    \frac{\partial J}{\partial \paramnot}$.   Taking steps in the
direction which optimizes this estimate of the gradient is known as
stochastic gradient descent \cite{robbinsSGD}.

 We describe in Section
    \ref{sec:model_free} some model-free RL-based approaches that have
    worked in the inventory control problem.  These approaches typically
    use the policy-gradient theorem \cite{sutton2018reinforcement} to estimate a gradient.
    However, inspired by ARS (see
    Section \ref{sec:model_free}), we consider how to compute an estimate
    which has lower variance and hence may speed up the optimization.  This
    {\em direct backpropagation} will be discussed next (see Theorem \ref{theorem:diff}).


    %

    \subsection{Direct Backpropagation for Inventory Control}
    \label{sec:directbp}

    We can take a draw from our sample space $\mathcal{S}$ and generate a sequence of
    states and actions.  At each time $t$ the reward $R^i_t(\paramnot)$ is a continuous
    function of these random variables.  But we also can
    write $I^i_{t}$ as a function of $H_t$, and
    $\paramnot$ (the parameters of the policy). We will denote this below
    as $I^i_t := \mathcal{T}_t(H_t, \paramnot)$, where f is the sub-differentiable function defined in
    Equation \eqref{eqn:transitionper2}. 

    Glueing these all together allows us to
     think of the sum of the rewards (which is a random variable) as a
     continuous function of $\theta$.  We can take the derivative of this function and
     call it $\widehat{\frac{\partial J}{\partial
\paramnot}}$. This then is our stochastic partial derivative. We make the following assumption
on our policy $\policynot$

\begin{assumption}[Well behaved policy] Suppose $\forall t$, a policy
$\policynot: \overset{|{\cal A}|}{\bigotimes}\left(\overset{t-1}{\bigotimes}\mathcal{S}\right) \rightarrow \mathbb{R}$
        satisfies the following conditions :
\begin{itemize}
\item $\policynot(H_t)$ is an absolutely continuous function of $\theta$ for
almost all $H_t$,
\item $\frac{\partial \policynot(H_t)}{\partial\theta}$ is ``locally
        integrable'' that is, for all compact neighborhoods $N_\epsilon (\theta_0)$ contained in $\bigparamnot$:

\begin{equation*}
        \int_{N_\epsilon(\theta_0)} \rewardexp \left[ \left\lvert \frac{\partial \policynot(H_t)}{\partial\theta}\right\rvert \right] d\theta < \infty
\end{equation*}

\end{itemize}

Then we say that $\policynot$ is \textit{well behaved}.
\end{assumption}

    \begin{theorem}
    \label{theorem:diff}
If $\forall i$ and $\forall t$, $\pi^i_{\theta,t}(H_t)$ is well behaved in $\theta$ then
    \begin{displaymath}
            \frac{\partial J(\theta)}{\partial\theta} = \frac{1}{|\mathcal{A}|}\sum_{i \in \mathcal{A}}\sum^T_{t=0} \rewardexp\left[ \frac{\partial
            R^i_{t}(H_t, \paramnot)}{\partial \theta}\right] .
    \end{displaymath}
    \end{theorem}

    \begin{proof} ~
{\bf Proof:}
Fix the values of the random variables $(s_t^i)_{i\in \{1,\ldots,|{\cal A}|\},t\in\{1,\ldots,t\}}$.  The
action at time $1$ is a well behaved function of $\paramnot$ by
assumption.   The inventory level at the end of the period is linear in
the action and so is also a well behaved function.  The policy at
time $2$ depends on $\paramnot$ both directly and through the inventory
levels, both in a well behaved fashion and since $I_1^i$ is a well
behaved function of $\paramnot$, we have $\policynot(H_2)$ is well
behaved.  Proceeding in this fashion, we see that $\policynot(H_t)$ is
well behaved in $\paramnot$.  The reward function is a piecewise linear
function of the inventory and the actions at time $t$.  So it too is
well behaved.  Hence by the Fundamental Theorem of Calculus for
Lebesgue integration we have the
result.





\hfill $\blacksquare$

    \end{proof}

    Using modern deep learning libraries, such as PyTorch \cite{NEURIPS2019_9015}, the function $J_T(\paramnot)$ can be constructed dynamically via simulated rollouts. Algorithm \ref{alg:dbp} in Section \ref{sec:alg-details} describes the backpropagation procedure for the inventory control problem.

    The extension to infinite horizon problems is straightforward -- we just need to add a differentiable value function that can be applied to the terminal state. In the inventory control application (Section \ref{sec:results}) we use a fixed percentage times the cost of the remaining inventory as the terminal value, although one could replace this with a differentiable estimate of the state-value function $V^{\paramnot, M}_T$.

    \subsection{Model-Free Reinforcement Learning} \label{sec:model_free}

    We compare against several canonical model-free approach - including the
    Asynchronous Advance Actor Critic \cite{a3c} (or {\bf A3C}), the {\bf Soft
    Actor-Critic} \cite{haarnoja2018soft} (or {\bf SAC}), and Augmented Random
    Search \cite{ars} (or {\bf ARS}) algorithms. Typically, for Actor-Critic
    style methods, we share parameters between the action and value networks and
    use a Gaussian distribution for the output. More details about the network
    structures can be found in Appendix \ref{app:network}.

    \section{Learnability of Periodic Review Inventory Management IDPs}
    \label{sec:learnability}

    In this section, we describe a collection of results that illustrate the difference
    between our problem and the classical RL setting. Since a large part of the state
    space is a random variable that is unaffected by the agent's actions - we are
    able to learn efficiently, even in the case where some random variables are
    not directly observed. To obtain these results, we will make certain
    assumptions that go beyond those outlined in Section \ref{sec:formulation}.

    Assume that products are sampled independently. In this case, the measure
    $\measurenot$ will be become a product measure $\measurenot = \prod_i \measurenot^i$. Consider for product
    $i$ denote the ``history'' vector upto time $t$,
    \begin{displaymath}
      H^i_t := (s^i_1, \ldots, s^i_{t-1}).
    \end{displaymath}

     Note here that we use the same parameters $\paramnot$ for the policy of each
product. It is helpful at this point to note that our reward for product $i$ at
time $t$ will only depend on the ``history'' vector for that product, i.e.,
$R^i_t := R(H^i_t, \paramnot)$, which we abbreviate to $R^i_t(\paramnot)$ for
notational convenience.

     The generic differential control problem for a single product $i$
     can then be defined as:
     \begin{equation}
             J^i_T(\paramnot) := \rewardexp \left[\sum^T_{t=0} \gamma^t
     R^i_t(\paramnot) \right]
     \end{equation}

     Note that constraints on the action space can be pushed inside the
     reward function $R$ so they are omitted from this formulation.
     Now, note that for the inventory control problem, our reward function is
     linear in the products being managed, so our objective function becomes
         \[
             J_T(\paramnot) = \productsumnot J^i_T(\paramnot)
         \]

    \subsection{Fully Observed Case}
    \label{sec:fulllearnability}

In this section, we further assume that that the policy for product $i$ upto
time $t$ only depends on the history upto time $t$ of that
product itself $\policynot: \overset{t-1}{\bigotimes}~{\cal S} \rightarrow
     \mathbb R$, so
    \begin{displaymath}
    a^i_t :=  \policynot(H^i_t) = \policynot(s^i_1,\ldots s^i_{t-1}) .
    \end{displaymath}

In the fully observed case, note that the ``history'' vector
$H^i_T := (s^i_1, \ldots, s^i_T)$ constitutes a single sample for each product
$i$. We can then define, from the sampled reward, for a product $i$ at time $T$:

      \begin{equation*}
              \hat{J}^i_T(\paramnot) := \sum^T_{t=0} \gamma^t
  R^i_t(\paramnot)
      \end{equation*}

      We can then denote the empirical objective function by:

      \begin{equation*}
              \hat{J}_T(\paramnot) := \productsumnot \hat{J}^i_T(\paramnot)
      \end{equation*}

     Consider the case where all of our random
     variables are fully observed, and so historical data consists of samples
     from the true data generating process \textbf{for any} policy.

    \begin{theorem}
    \label{theorem:hoeffding}
    Consider a finite parameter set $\Theta$, with $\theta \in
\Theta$. Given any $\delta \in (0, 1)$, with probability greater than $1-\delta$,
we have that
    \begin{displaymath}
      \forall \theta \in \Theta: |J(\theta) - \hat{J}(\theta)| \leq
      (R^{\text{max}} - R^{\text{min}}) T
      \sqrt{\frac{\log\left(\frac{2 |\Theta|}{\delta}\right)}{2|\mathcal{A}|}}
    \end{displaymath}
    \end{theorem}

    \begin{proof}~
    {\bf Proof:} We will show that
    \begin{displaymath}
    \mathbb{P}(\exists \theta^{*} \in \Theta:
    |J_T(\theta^{*}) - \hat{J}_T(\theta^{*})| > \epsilon) \leq \delta
    \end{displaymath}

    Using the union bound, we obtain:

    \begin{displaymath}
    \mathbb{P}(\exists \theta^{*} \in \Theta: |J_T(\theta^{*}) - \hat{J}_T(\theta^{*})| > \epsilon) \\
    \leq \sum_{\theta \in \Theta} \mathbb{P}(|J_T(\theta) - \hat{J}_T(\theta)| > \epsilon) \\
    \end{displaymath}

    From which follows:

    \begin{align*}
      \sum_{\theta \in \Theta} \mathbb{P}(|J_T(\theta) - \hat{J}_T(\theta)| > \epsilon)
      &= \sum_{\theta \in \Theta} \mathbb{P}\left( \frac{1}{|\mathcal{A}|}
      \left|\sum_{i \in \mathcal{A}} J^i_T(\theta) - \sum_{i \in \mathcal{A}}
      \hat{J}^i_T(\theta)\right| > \epsilon \right)\\
      &\leq  \sum_{\theta \in \Theta}
      2e^{-\frac{2|\mathcal{A}|\epsilon^2}{T^2 (R^{\text{max}} - R^{\text{min}})^2}} \\
      &\leq 2 |\Theta| e^{-\frac{2|\mathcal{A}|\epsilon^2}{T^2 (R^{\text{max}} - R^{\text{min}})^2}}
    \end{align*}

    Where the first equality follows from definition,
    the second comes from an application of Hoeffding's Lemma.
    Setting the right hand side to $\delta$ and solving gives the result.
    \end{proof}

    \hfill $\blacksquare$

    Theorem \ref{theorem:hoeffding} illustrates that under the fully observed case,
    we can efficiently backtest \textbf{any} policy using historical data.
    Note that this represents an exponential improvement over the sample
    complexity of a generic RL problem \cite{agarwal2019reinforcement}.\footnote{
    In generic RL problems, assuming two actions and a uniform measure,
    the sample complexity can be as large as $\min(|\Theta|,2^{T})$.} Note that
    a similar bound can also be found in \cite{sinclair2022hindsight}.

\subsection{Partially Observed Case}
\label{sec:partiallearnability}

As noted in Section \ref{subsec:censoring},
the nature of our IDP is such that some historical data is naturally censored
and must be imputed in order to recover a full trajectory of the ``history''
vector $H^i_T$ for any product $i$.

Following the notation in \cite{nikulin2011non}, denote for product $i$ at
time $t$, by $\delta^i_t$ the indicator function of whether the value of the
state at that time is observed or not. Note that $\delta^i_t$
is a random variable in itself, that may not be "exogenous" (i.e. it
maybe a function of the decisions taken by some policy) - for example:

\paragraph{Illustrative Example: Censored Demand } In order to map this
framework to our IDP, we consider the censoring of demand outlined
in Section \ref{subsec:censoring}. In this case, the indicator function
$\delta^i_t$ can be defined as follows:
\begin{displaymath}
  \delta^i_t := \mathbbm{1}_{D^i_t < I^i_{t_{-}}}
\end{displaymath}

Thus, the observed history vector $H^i_{T, O}$ can be defined in the following way:
\begin{displaymath}
  H^i_{T, O} := ((s^i_1  \delta^i_1, \delta^i_1), \ldots,
                (s^i_T  \delta^i_T, \delta^i_T))
\end{displaymath}

and the missing portion of the history $H^i_{T, M}$ will correspond to:
\begin{displaymath}
  H^i_{T, M} := ((s^i_1, \delta^i_1), \ldots,
                (s^i_T, \delta^i_T))
\end{displaymath}

As shown in Section \ref{sec:acdemand}, we can use additional context that maybe
available to us in order to {\it forecast} the missing portions of our
historical  data, i.e., we might also use additional context $\covariatenot^i_t
\in \mathbb{R}^{d}$ that might be available for any particular product $i$ at
time $t$. Examples of such a context might include the number of website glance
views the product received in that time period, or an embedding of the product
description, or other covariates such as macro-economic data.

Let
\begin{displaymath}
  X^i_T := (x^i_1, \ldots, x^i_T) \in {\cal X}
\end{displaymath}
denote the entire vector of (observed) context for product $i$ upto time $T$.

In this section, we further assume that that the policy for product $i$ upto
time $t$ only depends on the observed history and covariates upto time $t$ of that product
itself, i.e., we assume that $\policynot:  \overset{t-1}{\bigotimes}~{\cal S}
\times \overset{t-1}{\bigotimes}\{0, 1\}
\times \overset{t-1}{\bigotimes}~{\cal X}\rightarrow \mathbb R$.
\\
\\
The action $a^i_t$ then becomes:
\begin{displaymath}
  a^i_t :=  \policynot(H^i_{T, O}, X^i_t).
\end{displaymath}

With these definitions, we outline our key assumption - the existence of an
{\it accurate forecast}.


\begin{assumption}[Accurate Forecast]
  \label{assumption:forecast}
Suppose that $H^i_{T, O}$ is the observed ``history'' vector and $X^i_T$ is the
vector of covariates for product $i$. Consider the following probability distribution
$\mathbb{P}^i_M$:
\begin{displaymath}
    \mathbb{P}^i_M  := \mathbb{P}^i(H^i_{T, M} | H^i_{T, O}, X^i_T)
\end{displaymath}

Now, denote by $\hat{\mathbb{P}}^i_M$, the following distribution:
\begin{displaymath}
  \hat{\mathbb{P}}^i_M  := \hat{\mathbb{P}}^i(H^i_{T, M} | H^i_{T, O}, X^i_T)
\end{displaymath}

Let us assume that, almost surely, we can
achieve the following supervised learning error guarantee on this sample:\footnote{It is
  straightforward to modify this assumption and to rewrite the theorem for this
  to hold with high probability}
\[
\frac{1}{|\mathcal{A}|} \sum_{i \in \mathcal{A}} \tv{\mathbb{P}^i_M,
\hat{\mathbb{P}}^i_M}  \leq \epsilon_{sup}
\]

Then we will say that $\hat{\mathbb{P}}^i_M$ is an accurate forecast of
$\mathbb{P}^i_M$.\footnote{Note that the \textit{sup} in $\epsilon_{sup}$ refers to
a supervised learning error.}
\end{assumption}

We will need to estimate our $J_T(\theta)$ based on the observed
variables, the product specific covariates and our (conditional) probabilistic model
$\hat{\mathbb{P}}^i_M$. With respect to a
single product define:

\begin{displaymath}
\tilde{J}^i_T(\theta) :=
\mathbb{E}^{\widehat{\mathbb{P}}^i_M}\left[\sum^T_{t=0} \gamma^t
  R^i_t(\paramnot) \middle| H^i_{T, O}, X^i_T\right]
\end{displaymath}
and define
\begin{displaymath}
  \tilde{J}_T(\theta) := \frac{1}{|{\cal A}|} \sum_{i \in \mathcal{A}}
\tilde{J}^i_T(\theta)
\end{displaymath}
~
\begin{theorem}
  \label{theorem:censoredlearnability} Suppose that $H^i_{T, O}$ is an observed
sample and $X^i_T$ is the realized covariate vector for product $i$ upto time $T$.
Assume we have an accurate forecast, as defined in Assumption \ref{assumption:forecast},
$\hat{\mathbb{P}}^i_M$ for the distribution $\mathbb{P}^i_M$. Given any
$\delta \in (0, 1)$, with probability greater than $1-\delta$,
we have that:
  \begin{displaymath}
    |J_T(\paramnot) - \tilde{J}_T(\paramnot)| \leq
    T (R^{\text{max}} - R^{\text{min}}) \left(\epsilon_{\textrm{sup}} +
    \sqrt{\frac{\log\left(\frac{2 |\Theta|}{\delta}\right)}{2|\mathcal{A}|}}
    \right)
  \end{displaymath}
\end{theorem}

    \begin{proof} ~
      {\bf Proof:} To aid with the proof, we first begin by defining the quantity:
      \begin{displaymath}
      \bar{J}^i_T(\theta) :=
      \mathbb{E}^{\mathbb{P}^i_M}\left[\sum^T_{t=0} \gamma^t
        R^i_t(\paramnot) \middle| H^i_{T, O}, X^i_T\right]
      \end{displaymath}

      We have that:
      \begin{align*}
        |J_T(\paramnot) - \tilde{J}_T(\paramnot)| &=
        \frac{1}{|\mathcal{A}|} \left|\sum_{i \in \mathcal{A}} J^i_T(\theta)
        - \sum_{i \in \mathcal{A}} \tilde{J}^i_T(\theta)
        \right|
\\
        &= \frac{1}{|\mathcal{A}|} \left| \sum_{i \in \mathcal{A}} \left(J^i_T(\theta)
        - \bar{J}^i_T(\theta) +
        \bar{J}^i_T(\theta) -
        \tilde{J}^i_T(\theta)\right) \right|
\\
        &\leq \uunderbrace{\frac{1}{|\mathcal{A}|} \left| \sum_{i \in \mathcal{A}} \left(J^i_T(\theta)
        - \bar{J}^i_T(\theta)\right) \right|}{A}+
        \uunderbrace{\frac{1}{|\mathcal{A}|}\left|\sum_{i \in \mathcal{A}}\left(
        \overline{J}^i_T(\theta)-
        \tilde{J}^i_T(\theta)
        \right)\right|}{B}
      \end{align*}

      Note that:
      \begin{displaymath}
        \mathbb{E}^{\mathbb{P}^i}\left[\bar{J}^i_T(\theta)\right] = J^i_T(\paramnot)
      \end{displaymath}

      Thus, we can apply Hoeffding's Lemma in the same way as
      Theorem \ref{theorem:hoeffding} to bound term A.

      For term B, we use the Simulation Lemma\footnote{The specific proposition
      (Proposition \ref{proposition:simulationlemma}) is recreated
      for exposition in Appendix \ref{app:proofs}}
      (see \cite{agarwal2019reinforcement}) to establish a bound.

      \begin{align*}
        \frac{1}{|\mathcal{A}|} \left| \sum_{i \in \mathcal{A}} \left(\overline{J}^i_T(\theta)-
        \tilde{J}^i_T(\theta)\right) \right|
                &\leq \frac{1}{|\mathcal{A}|} \sum_{i \in
                  \mathcal{A}} T(R^{\text{max}} - R^{\text{min}})
        \tv{\mathbb{P}^i_M,
         \hat{\mathbb{P}}^i_M}\\
                &\leq T(R^{\text{max}} - R^{\text{min}})\epsilon_{sup}
      \end{align*}

      Where the second inequality follows from Assumption \ref{assumption:forecast}.
      Combining these two terms gives the result.

    \end{proof}

   \hfill $\blacksquare$

   Note that the bound in Theorem \ref{theorem:censoredlearnability}
   closely mirrors the one in Theorem \ref{theorem:hoeffding} with an additional
   term that reflects the ``average'' error in the predictions of the imputation
   model.
   Thus, given an \textit{accurate forecast} of the unobserved
   random variables, the empirical (i.e. from historical data) value
   function $\tilde{J}_T(\paramnot)$ will be a good estimate of the
   true value function $J_T(\paramnot)$.

   \section{Baseline Policies}
   \label{sec:backtest}
   This section aims to describe the collection of baseline policies that we
   backtest existing RL algorithms and our proposed DirectBackprop against.

   \subsection{Oracle}
   \label{sec:oracle}
   We first benchmark against an oracle policy (\textbf{Oracle}), which
   represents an idealized scenario where the policy knows at time $t$ the
   future demand, price, cost and vendor lead-time of all products. This
   policy is not adapted to the data available at time $t$ and hence only
   serves as a benchmark to understand the behavior of the RL agents.
   Computing this policy involves solving the convex program given in
   Section \ref{sec:objective_function} \textit{ex-post}. This policy
   resembles the oracle defined in \cite{qi2020practical}.

   \subsection{Myopic Approximation} \label{sec:Nahmias}
   Following the approximate optimality results established by
   \cite{nahmias1979simple, morton1971}, we utilize a myopic
   approximation\footnote{We call this policy ``Myopic'' since it plans for only
   a single period at a time.} to the dynamic program. The optimal target
   inventory level is given by: \begin{equation} z^{*} =
   F^{-1}_{\mathbb{D}^i_t}\biggl(\frac{p^i_t - c^i_t}{p^i_t - c^i_t + c^i_t
   (1-\gamma)}\biggr), \end{equation} where $F^{-1}_{\mathbb{D}^i_t}$ denotes
   the inverse CDF (or quantile function) of demand for product $i$ at time $t$,
   and $z^{*}$ denotes the target inventory level. The order quantity then
   becomes $(z^{*} - \text{Supply}^i_t)^{+}$, where $\text{Supply}^i_t$ refers
   to the total inventory on hand and in flight.

   To generate $F^{-1}_{\mathbb{D}^i_t}$ we utilize a state-of-the-art quantile
forecaster \cite{mqcnn}, to produce 50th and 90th quantile forecasts for a fixed
planning horizon and then fit a gamma distribution to generate the
critical-ratio quantile.

   \subsection{Newsvendor} \label{sec:newsvendor}
   While \cite{nahmias1979simple, morton1971} produce effective myopic policies, they assume no continuation cost for the excess inventory. Following \cite{huh2009asymptotic, janakiraman2004lost} we consider a newsvendor problem, which optimizes over lead times and demand and a stationarity condition as the termination. This produces a classical OR baseline, where the target inventory 
   is given by:
   \begin{equation}
           z^{*, i}_t = \min_{y \geq 0} \rewardexp \biggl[ p^i_t \biggl(\sum^{t+v}_{s=t}\demnot^i_{s} - y\biggr)^{+} + (c^i_t + h) \times \biggl(y - \sum^{t+v}_{s=t}\demnot^i_{s}\biggr)^{+} \biggr] \label{eqn:newsvendor}
   \end{equation}
   where $h$ denotes a salvage value for the left over inventory. The order quantity then becomes $(z^{*, i}_t - \text{Supply}^i_t)^{+}$, where $\text{Supply}^i_t$ refers to the total inventory on hand and in flight.

   \subsection{Planning Horizon Normalized Policy} \label{sec:phn}

   Order up to policies implicitly assume that sales during the first vendor lead time period are fully backlogged. In order to mitigate this effect we improve the newsvendor base-stock policy in Section \ref{sec:newsvendor} by absorbing extra information about the vendor lead time (denoted by the random variable $\vltnot$) distribution.

  We normalize the target inventory level given by Equation \eqref{eqn:newsvendor} by the mean of the vendor lead time (for product $i$ at time $t$) to get the buying amount for each time step (review period), and multiply it by the median of the vendor lead times across all products in order to approximate the amount required on-hand to fulfill demand for a single lead time.

   Consider the target level set by the policy in Equation \eqref{eqn:newsvendor} above. Let the vendor lead time for the product $i$ at time $t+1$ be $v^i_{t+1}$, and assume $\rewardexp[v^i_{t+1} | \mathcal{F}_t] > 0$. The Planning Horizon Normalized (PHN) policy would order:
   \begin{equation}
       a(PHN)^i_t := \frac{z^{*, i}_t}{\rewardexp[v^i_{t+1} | \mathcal{F}_t]}.
   \end{equation}

   For our policies, we use $\text{Median}(\vltnot) * a(PHN)^i_t$ as a Planning-Horizon normalized baseline to the newsvendor.

    \section{Empirical Results}
    \label{sec:results}
    We evaluate the following baseline and RL policies:
    \begin{itemize}
        \item \textbf{Oracle} (Section \ref{sec:oracle}): the agent optimizes the order quantity by looking at the future values of different random variables.
        \item \textbf{Myopic} (Section \ref{sec:Nahmias}): the agent orders the amount given by the myopic policy from \cite{nahmias1979simple} over a fixed vendor lead time horizon.
        \item \textbf{Newsvendor} (Section \ref{sec:newsvendor}): the agent orders the amount given by the newsvendor policy over the expected planning horizon.
        \item \textbf{2.5 $\times$ PHN} (Section \ref{sec:phn}): the agent orders $2.5$ times the normalized amount given by the newsvendor problem. $2.5$ is the median value of all products' lead times.
        \item RL policies (Section \ref{sec:rl_policies} and Section \ref{sec:directbp}): we choose \textbf{ARS}, \textbf{SAC}, \textbf{A3C} and \textbf{DirectBackprop} to backtest. They are the top 4 performing algorithms across our experiments.
    \end{itemize}
    See Appendix \ref{app:network} for details on the network architectures used for each of the RL policies.

    \subsection{Data}
    \label{sec:data}

    Due to the lack of appropriate publicly available or baseline datasets, we use a weekly dataset of 80,000 sampled products from a single marketplace. Our dataset begins in April 2017 and ends in August 2019. We use the first 104 weeks for training, and evaluate out of sample on the last 19 weeks.

    Missing data points are filled with the previous time periods value, so as not to leak any information. We perform a number of preprocessing operations to the data. We standardize input features across all products for SAC and ARS (see Section \ref{sec:rl_policies}). For A3C, we standardize per product but across time. In addition, for A3C we also apply a Box-Cox transform to price, cost and demand in order to make long-tail distributions look more like a Gaussian distribution. Unlike the other RL algorithms, DirectBackprop did not require any feature preprocessing aside from the scaling factors. See Appendix \ref{app:featurization} for more details. The baseline algorithms also require no feature preprocessing.

    \subsection{Training}
    During the training phase, various RL agents are experimented with using RLlib \cite{liang2017rllib}. We also implemented a differentiable simulator using PyTorch \cite{NEURIPS2019_9015} which we used to train the {\bf DirectBackprop} policy. All algorithms were trained using an EC2 p3.16xlarge instance.

    \subsection{Results}
    \label{sec:results-sim}
    All three initialization cases laid out in Section \ref{sec:initterm} are evaluated. We used a terminal value (Equation \eqref{eqn:terminal_value}) of $0$.\footnote{we also investigated different termination options included in Equation \eqref{eqn:terminal_value} and saw no substantial performance difference.} Note that the testing data used is ``out of sample'' for the RL algorithms, in that all the data seen by the algorithms occurs in the chronological past. The ultimate goal of an RL agent is to maximize the expected discounted cumulative reward, which in this case, is approximated by Equation \eqref{eqn:approx_reward}. We summarize the performances in terms of cumulative reward at end of the testing period in Table \ref{table:performance} for a concrete comparison.

    \begin{table}
        \caption{Discounted cumulative reward normalized by the value
    of \textbf{Oracle} at end of the testing period under different
    initializations. ``On-Hand'' initialization means historic inventory
    positions were used and ``From NV'' indicates a newsvendor policy was
    rolled out up until the start of the test period, and then the
    evaluation policy rolled out thereafter.
    }
    \centerline{
        \begin{tabular}{l c c c}
            \toprule
                           & \multicolumn{3}{c}{Initialization}                             \\
            \cmidrule{2-4}
            Policy         & Zero                               & On-Hand     & From NV     \\
            \midrule
            Oracle         & 100.0                              & 100.0       & 100.0       \\
            Myopic         & 13.57                              & 52.05       & 56.48       \\
            Newsvendor     & 29.26                              & 61.62       & 65.48       \\
            2.5*PHN        & 47.39                              & 71.17       & 74.25       \\
            A3C            & 41.59                              & 67.80       & 71.20       \\
            SAC            & 41.13                              & 67.07       & 70.53       \\
            ARS            & 55.92                              & {\bf 76.09} & 78.06       \\
            DirectBackprop & {\bf 58.39}                        & 74.12       & {\bf 79.27} \\
            \bottomrule
        \end{tabular}}
        \label{table:performance}
    \end{table}

    As expected, the \textbf{Oracle} achieves the highest cumulative reward at end of the testing period. It outperforms other policies at all time points when the inventory is initialized with on-hand units or the amount given by the \textbf{Newsvendor} policy. \textbf{2.5 $\times$ PHN} is the second best baseline policy overall. This is primarily because it solves a critical issue with base-stock policies, that they order slightly extra by looking at the entire planning horizon as opposed to a period corresponding to a single vendor lead time. The two major base-stock policies \textbf{Newsvendor} and \textbf{Myopic} perform similarly - with \textbf{Newsvendor} showing a slight advantage. All RL policies are competitive with or superior to the baseline policies, and the performance is stable across different initializations. This verifies that RL indeed can be a good fit for our problem.

    Using the Newsvendor-determined inventory initial state, the best RL policy ({\bf DirectBackprop}) represents a 23\% improvement over the baseline policy (\textbf{Newsvendor}). As to how this backtest improvement translates into real world, one has to keep in mind that certain aspects of the reality are not simulated. Although the simulation plays back historical data, for vendor lead times we had to resort using samples from predictive distributions instead of actuals, vendor fill rates are not modeled (assumed to be 100\%), and there are no constraints to order quantities or order intervals. However, all policies operate under these same conditions.\footnote{It is important to caveat that systematic deviations of the Gym from reality may render relative improvements moot. For example, a policy that always buys more may perform better than another policy in the presence of fill rates - simply by virtue of holding more buffer inventory for weeks with low fill. This can be mitigated by constantly calibrating the Gym against reality.}

    \begin{figure}[htb!]
        \centerline{\includegraphics[width=0.7\linewidth]{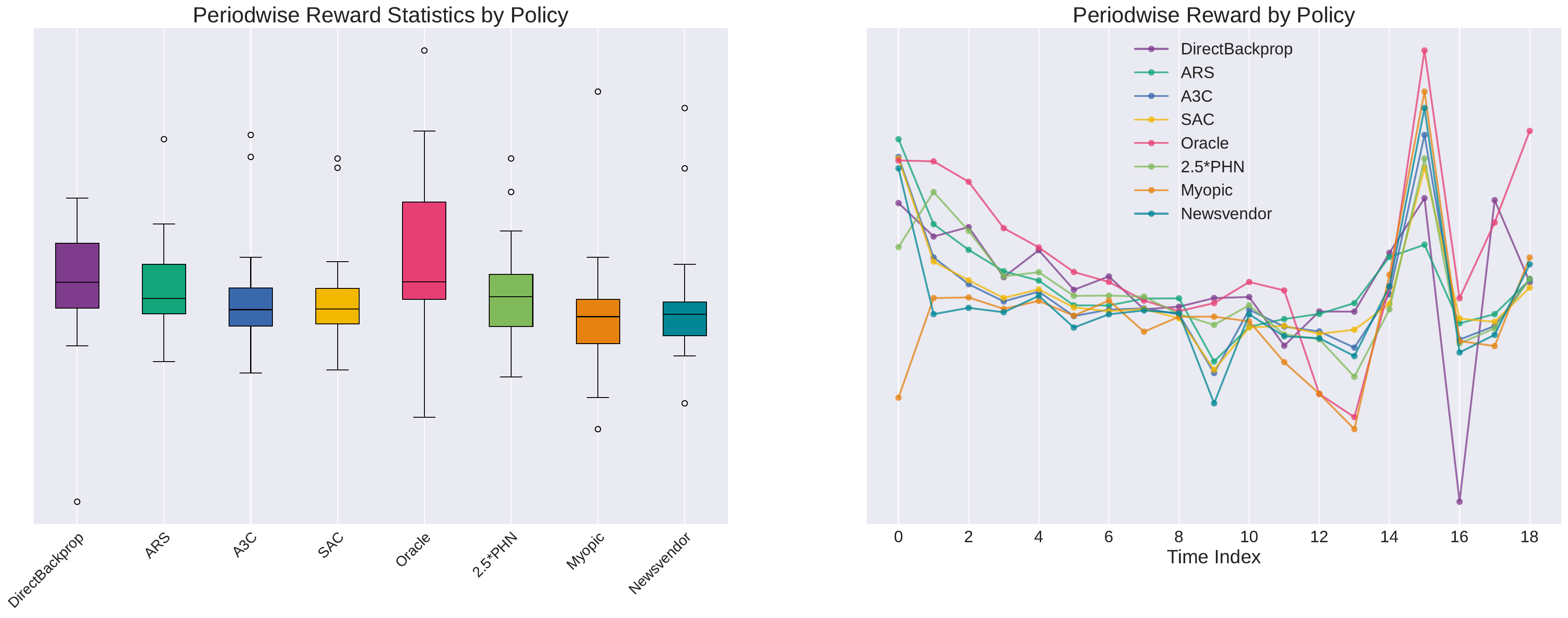}}
        \caption{Behavior of baseline and RL-policies during the test period, when the initial inventory position is given by the \textbf{Newsvendor} policy. The figure on the right depicts the period-wise reward (across products) for the different RL policies. The plot on the left depicts the period-wise reward statistics, where we see that DirectBackprop has a higher mean and median reward than all other non-oracle policies.}
        \label{fig:tip_initialized_graphs}
    \end{figure}

    \begin{figure}[htb!]
        \centering
        \includegraphics[width=0.7\linewidth]{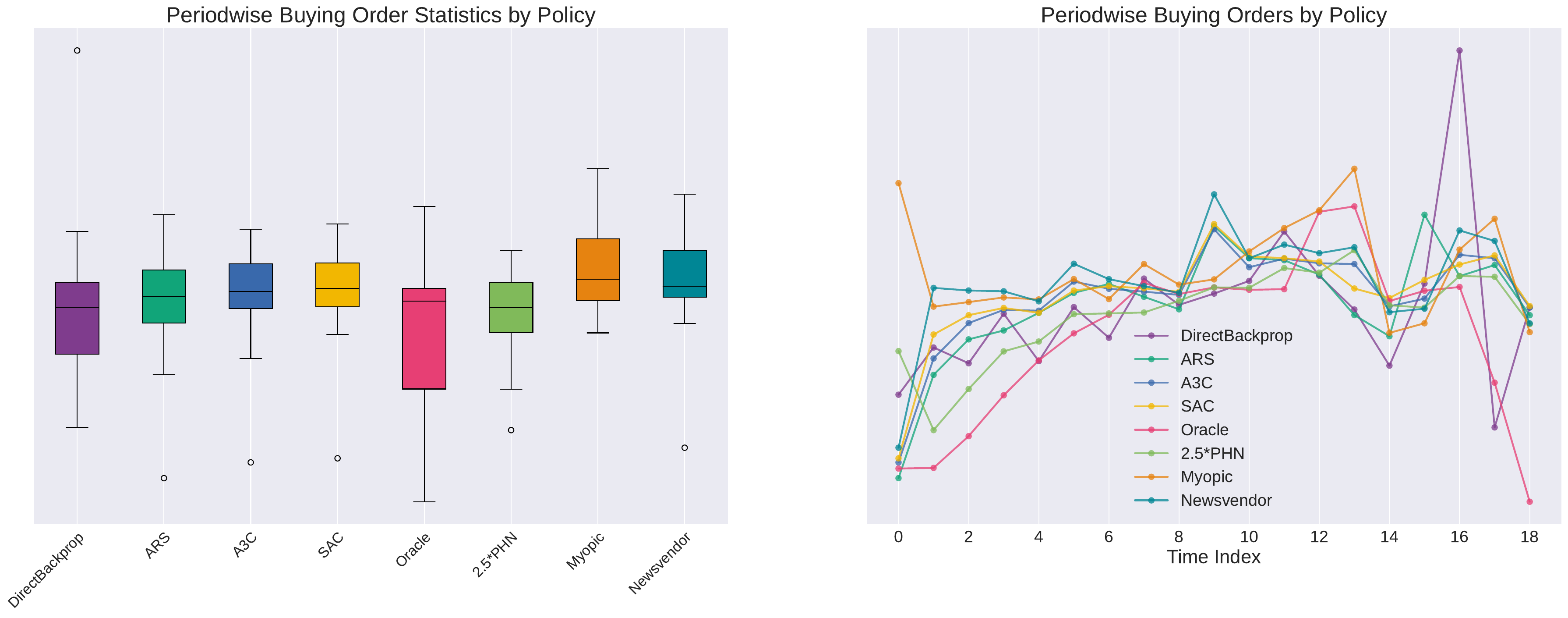}
        \caption{Behavior of baseline and RL-policies during the test period, when the initial inventory position is given by the \textbf{Newsvendor} policy. The figure on the right depicts the period-wise total orders (across products) for the different RL policies. The plot on the left depicts the period-wise total order statistics, where we see that DirectBackprop is able to order less while maintaining a higher reward than other non-oracle policies.}
	\label{fig:tip_initialized_graphs_orders}
\end{figure}

\begin{figure}[htb!]
	\centering
	\includegraphics[width=0.7\linewidth]{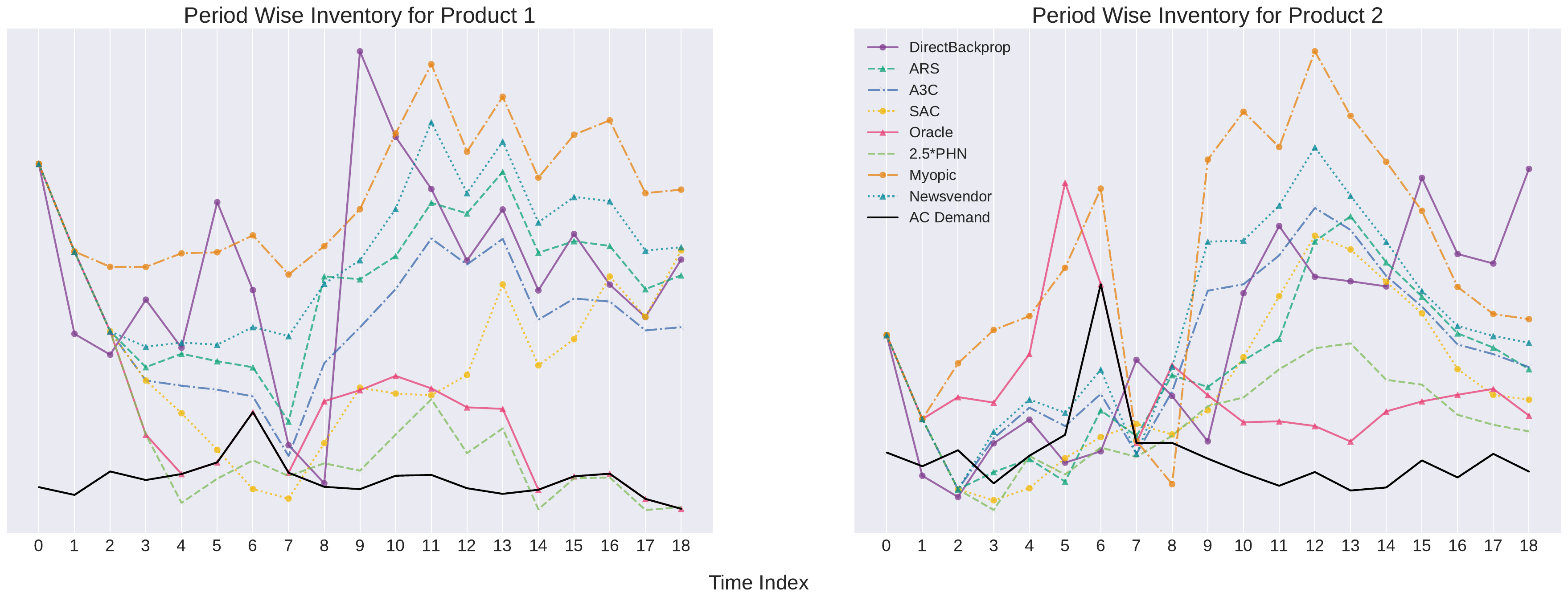}
	\caption{Inventory position for two random products at each time point \textit{after} supply arrives and \textit{before} demand is fulfilled. The solid black line is the availability corrected demand defined in Equation \eqref{eqn:ac_demand}; exploration is turned off for RL algorithms. The plot shows that the RL algorithms are able to generate reasonable inventory curves.}
	\label{fig:random_products}
\end{figure}

Figures \ref{fig:tip_initialized_graphs} and \ref{fig:tip_initialized_graphs_orders} show graphs of certain metrics over the test period. We see that DirectBackprop outperforms all non-oracle policies despite ordering less (and therefore holding less inventory).

Finally, the inventory level over time for two products is presented against the AC Demand (Equation \eqref{eqn:ac_demand}) in Figure \ref{fig:random_products}. We randomly chose these two products for the purpose of illustration. The inventory position is measured \textit{after} supply has arrived and \textit{before} demand has been fulfilled. The trends are consistent with the observations from previous figures.

\subsection{Real-World Deployment}
In this section we present results from a real-world deployment of the Direct-Backprop agent. The deployment was for 10K products over a period of 26 weeks. We randomized these products into a Treatment (receiving Direct-Backprop buy quantities) and a Control (receiving Newsvendor policy buy quantities) group. The Control group was the current production system used by one of the
largest Supply Chains in the world. The Treatment group was able to significantly reduce inventory (by $\sim12\%$) without losing any excess revenue (statistically insignificant difference from 0\%). We present the results in Table \ref{table:realworld}. Based on these results, we believe RL is a promising line of research for inventory control in the real world.

\begin{table*}[tb]
	\centering
	\begin{tabular}{c c}
		\toprule
		Metrics         & \% change                               \\
		\midrule
		Inventory Level & $\textbf{-12} \pm 6$                            \\
		Revenue         & $\sim$                              \\
		\bottomrule
	\end{tabular}

  \vspace{2mm}
  \caption{Inventory Level (\% change in Units) and Revenue (\% change in \$) for a randomized
real-world deployment of the Direct Backpropagation agent. The Revenue metric was statistically indistinguishable from $0$.}
\label{table:realworld}
\end{table*}

\section{Conclusion}

In this paper, we developed a novel approach for turning historical data into a simulator - showing that Deep Reinforcement Learning is a good candidate solution for solving historically intractable dynamic programming problems in inventory control. We also presented a novel inventory management system that handles a discrete-time periodic review inventory management system with correlated demand, price matching, random lead times and lost sales at each decision epoch. Finally, we proposed a model based reinforcement learning procedure (DirectBackprop), which leverages information we have available about the system dynamics and allows us to outperform both classical policies and modern model-free RL methods.

Further applications of model-based RL to the inventory management problem - such as learning a model for the transition dynamics or the incorporation of cross-product constraints (see Equations \eqref{eqn:arrivalconst}-\eqref{eqn:unitconst} in Appendix \ref{app:objective}) - remains an exciting avenue of research.

\clearpage
\bibliographystyle{abbrv}
\bibliography{RL}

\begin{thebibliography}{10}

\bibitem{abbeel2005exploration}
P.~Abbeel and A.~Y. Ng.
\newblock Exploration and apprenticeship learning in reinforcement learning.
\newblock In {\em Proceedings of the 22nd international conference on Machine
  learning}, pages 1--8, 2005.

\bibitem{agarwal2019reinforcement}
A.~Agarwal, N.~Jiang, S.~M. Kakade, and W.~Sun.
\newblock Reinforcement learning: Theory and algorithms.
\newblock {\em CS Dept., UW Seattle, Seattle, WA, USA, Tech. Rep}, pages 10--4,
  2019.

\bibitem{balaji2019orl}
B.~Balaji, J.~Bell-Masterson, E.~Bilgin, A.~Damianou, P.~M. Garcia, A.~Jain,
  R.~Luo, A.~Maggiar, B.~Narayanaswamy, and C.~Ye.
\newblock Orl: Reinforcement learning benchmarks for online stochastic
  optimization problems.
\newblock {\em arXiv preprint arXiv:1911.10641}, 2019.

\bibitem{bellman1955optimal}
R.~Bellman, I.~Glicksberg, and O.~Gross.
\newblock On the optimal inventory equation.
\newblock {\em Management Science}, 2(1):83--104, 1955.

\bibitem{besbes2013implications}
O.~Besbes and A.~Muharremoglu.
\newblock On implications of demand censoring in the newsvendor problem.
\newblock {\em Management Science}, 59(6):1407--1424, 2013.

\bibitem{bijvank2011lost}
M.~Bijvank and I.~F. Vis.
\newblock Lost-sales inventory theory: A review.
\newblock {\em European Journal of Operational Research}, 215(1):1--13, 2011.

\bibitem{brockman2016openai}
G.~Brockman, V.~Cheung, L.~Pettersson, J.~Schneider, J.~Schulman, J.~Tang, and
  W.~Zaremba.
\newblock Openai gym.
\newblock {\em arXiv preprint arXiv:1606.01540}, 2016.

\bibitem{cooper2006models}
W.~L. Cooper, T.~Homem-de Mello, and A.~J. Kleywegt.
\newblock Models of the spiral-down effect in revenue management.
\newblock {\em Operations research}, 54(5):968--987, 2006.

\bibitem{dada2007newsvendor}
M.~Dada, N.~C. Petruzzi, and L.~B. Schwarz.
\newblock A newsvendor’s procurement problem when suppliers are unreliable.
\newblock {\em Manufacturing \& Service Operations Management}, 9(1):9--32,
  2007.

\bibitem{rlsmart}
T.~K. Das, A.~Gosavi, S.~Mahadevan, and N.~Marchalleck.
\newblock Solving semi-markov decision problems using average reward
  reinforcement learning.
\newblock {\em Management Science}, 45(4):560--574, 1999.

\bibitem{efroni2022sample}
Y.~Efroni, D.~J. Foster, D.~Misra, A.~Krishnamurthy, and J.~Langford.
\newblock Sample-efficient reinforcement learning in the presence of exogenous
  information.
\newblock {\em arXiv preprint arXiv:2206.04282}, 2022.

\bibitem{efroni2022sparsity}
Y.~Efroni, S.~Kakade, A.~Krishnamurthy, and C.~Zhang.
\newblock Sparsity in partially controllable linear systems.
\newblock In {\em International Conference on Machine Learning}, pages
  5851--5860. PMLR, 2022.

\bibitem{eisenach2020}
C.~Eisenach, Y.~Patel, and D.~Madeka.
\newblock Mqtransformer: Multi-horizon forecasts with context dependent and
  feedback-aware attention.
\newblock {\em arXiv preprint arXiv:2009.14799}, 2020.

\bibitem{feng2006optimality}
Q.~Feng, H.~Yan, H.~Zhang, and S.~Sethi.
\newblock Optimality and nonoptimality of the base-stock policy in inventory
  problems with multiple delivery modes.
\newblock {\em Journal of Industrial and Management Optimization}, 2(1):19--42,
  2006.

\bibitem{giannoccaro2002inventory}
I.~Giannoccaro and P.~Pontrandolfo.
\newblock Inventory management in supply chains: a reinforcement learning
  approach.
\newblock {\em International Journal of Production Economics}, 78(2):153--161,
  2002.

\bibitem{gisjbrechts2021can}
J.~Gijsbrechts, R.~N. Boute, J.~A. Van~Mieghem, and D.~Zhang.
\newblock Can deep reinforcement learning improve inventory management?
  performance on dual sourcing, lost sales and multi-echelon problems.
\newblock {\em Manufacturing \& Service Operations Management}, 2021.

\bibitem{glasserman1995sensitivity}
P.~Glasserman and S.~Tayur.
\newblock Sensitivity analysis for base-stock levels in multiechelon
  production-inventory systems.
\newblock {\em Management Science}, 41(2):263--281, 1995.

\bibitem{gruen2002retail}
T.~W. Gruen, D.~S. Corsten, and S.~Bharadwaj.
\newblock Retail out of stocks: A worldwide examination of extent, causes, and
  consumer responses.
\newblock Technical report, Grocery Manufacturers of America, 2002.

\bibitem{haarnoja2018soft}
T.~Haarnoja, A.~Zhou, P.~Abbeel, and S.~Levine.
\newblock Soft actor-critic: Off-policy maximum entropy deep reinforcement
  learning with a stochastic actor, 2018.

\bibitem{hu2019chainqueen}
Y.~Hu, J.~Liu, A.~Spielberg, J.~B. Tenenbaum, W.~T. Freeman, J.~Wu, D.~Rus, and
  W.~Matusik.
\newblock Chainqueen: A real-time differentiable physical simulator for soft
  robotics.
\newblock In {\em 2019 International conference on robotics and automation
  (ICRA)}, pages 6265--6271. IEEE, 2019.

\bibitem{huh2009asymptotic}
W.~T. Huh, G.~Janakiraman, J.~A. Muckstadt, and P.~Rusmevichientong.
\newblock Asymptotic optimality of order-up-to policies in lost sales inventory
  systems.
\newblock {\em Management Science}, 55(3):404--420, 2009.

\bibitem{ingraham2018learning}
J.~Ingraham, A.~Riesselman, C.~Sander, and D.~Marks.
\newblock Learning protein structure with a differentiable simulator.
\newblock In {\em International Conference on Learning Representations}, 2018.

\bibitem{janakiraman2004lost}
G.~Janakiraman and R.~O. Roundy.
\newblock Lost-sales problems with stochastic lead times: Convexity results for
  base-stock policies.
\newblock {\em Operations Research}, 52(5):795--803, 2004.

\bibitem{janssen2013price}
M.~C. Janssen and A.~Parakhonyak.
\newblock Price matching guarantees and consumer search.
\newblock {\em International Journal of Industrial Organization}, 31(1):1--11,
  2013.

\bibitem{kaplan1970dynamic}
R.~S. Kaplan.
\newblock A dynamic inventory model with stochastic lead times.
\newblock {\em Management Science}, 16(7):491--507, 1970.

\bibitem{karlin1958inventory}
S.~Karlin.
\newblock Inventory models of the arrow-harris-marschak type with time lag.
\newblock {\em Studies in the mathematical theory of inventory and production},
  1958.

\bibitem{liang2017rllib}
E.~Liang, R.~Liaw, P.~Moritz, R.~Nishihara, R.~Fox, K.~Goldberg, J.~E.
  Gonzalez, M.~I. Jordan, and I.~Stoica.
\newblock Rllib: Abstractions for distributed reinforcement learning.
\newblock {\em arXiv preprint arXiv:1712.09381}, 2017.

\bibitem{madeka2018sample}
D.~Madeka, L.~Swiniarski, D.~Foster, L.~Razoumov, K.~Torkkola, and R.~Wen.
\newblock Sample path generation for probabilistic demand forecasting.
\newblock 2018.

\bibitem{maggiar2017joint}
A.~Maggiar and A.~Sadighian.
\newblock Joint inventory and revenue management with removal decisions.
\newblock Technical Report 3018984, SSRN, 2017.

\bibitem{ars}
H.~Mania, A.~Guy, and B.~Recht.
\newblock Simple random search provides a competitive approach to reinforcement
  learning.
\newblock {\em arXiv preprint arXiv:1803.07055}, 2018.

\bibitem{a3c}
V.~Mnih, A.~P. Badia, M.~Mirza, A.~Graves, T.~Lillicrap, T.~Harley, D.~Silver,
  and K.~Kavukcuoglu.
\newblock Asynchronous methods for deep reinforcement learning.
\newblock In {\em International conference on machine learning}, pages
  1928--1937, 2016.

\bibitem{moritz2018ray}
P.~Moritz, R.~Nishihara, S.~Wang, A.~Tumanov, R.~Liaw, E.~Liang, M.~Elibol,
  Z.~Yang, W.~Paul, M.~I. Jordan, et~al.
\newblock Ray: A distributed framework for emerging $\{$AI$\}$ applications.
\newblock In {\em 13th $\{$USENIX$\}$ Symposium on Operating Systems Design and
  Implementation ($\{$OSDI$\}$ 18)}, pages 561--577, 2018.

\bibitem{morton1971}
T.~E. Morton.
\newblock The near-myopic nature of the lagged-proportional-cost inventory
  problem with lost sales.
\newblock {\em Operations Research}, 19(7):1708--1716, 1971.

\bibitem{nahmias1979simple}
S.~Nahmias.
\newblock Simple approximations for a variety of dynamic leadtime lost-sales
  inventory models.
\newblock {\em Operations Research}, 27(5):904--924, 1979.

\bibitem{nikulin2011non}
M.~Nikulin, V.~Bagdonavi{\c{c}}ius, and K.~Julius.
\newblock Non-parametric tests for censored daata.
\newblock {\em HAL}, 2011, 2011.

\bibitem{oroojlooyjadid2016applying}
A.~Oroojlooyjadid, L.~Snyder, and M.~Tak{\'a}{\v{c}}.
\newblock Applying deep learning to the newsvendor problem.
\newblock {\em arXiv preprint arXiv:1607.02177}, 2016.

\bibitem{NEURIPS2019_9015}
A.~Paszke, S.~Gross, F.~Massa, A.~Lerer, J.~Bradbury, G.~Chanan, T.~Killeen,
  Z.~Lin, N.~Gimelshein, L.~Antiga, A.~Desmaison, A.~Kopf, E.~Yang, Z.~DeVito,
  M.~Raison, A.~Tejani, S.~Chilamkurthy, B.~Steiner, L.~Fang, J.~Bai, and
  S.~Chintala.
\newblock Pytorch: An imperative style, high-performance deep learning library.
\newblock In {\em NeurIPS}, 2019.

\bibitem{porteus2002foundations}
E.~L. Porteus.
\newblock {\em Foundations of stochastic inventory theory}.
\newblock Stanford University Press, 2002.

\bibitem{qi2020practical}
M.~Qi, Y.~Shi, Y.~Qi, C.~Ma, R.~Yuan, D.~Wu, and Z.-J.~M. Shen.
\newblock A practical end-to-end inventory management model with deep learning.
\newblock {\em Available at SSRN 3737780}, 2020.

\bibitem{MPC}
J.~B. Rawlings, D.~Q. Mayne, and M.~M. Diehl.
\newblock {\em Model Predictive Control: Theory, Computation, and Design}.
\newblock Nob Hill Publishing, 2nd edition, 2018.

\bibitem{reiman2004new}
M.~I. Reiman.
\newblock A new and simple policy for the continuous review lost sales
  inventory model.
\newblock Technical report, Bell Labs, 2004.

\bibitem{robbinsSGD}
H.~Robbins and S.~Monro.
\newblock A stochastic approximation method.
\newblock {\em The Annals of Mathematical Statistics}, 22, 1951.

\bibitem{scarf1959optimality}
H.~Scarf.
\newblock The optimality of (s,s) policies in the dynamic inventory problem.
\newblock Technical Report~11, Stanford University, 1959.

\bibitem{Schulman2016NutsBolts}
J.~Schulman.
\newblock Nuts and bolts of deep rl research.
\newblock NIPS Deep RL Workshop, 2016.

\bibitem{PPO}
J.~Schulman, F.~Wolski, P.~Dhariwal, A.~Radford, and O.~Klimov.
\newblock Proximal policy optimization algorithms.
\newblock {\em CoRR}, abs/1707.06347, 2017.

\bibitem{sinclair2022hindsight}
S.~R. Sinclair, F.~Frujeri, C.-A. Cheng, and A.~Swaminathan.
\newblock Hindsight learning for mdps with exogenous inputs.
\newblock {\em arXiv preprint arXiv:2207.06272}, 2022.

\bibitem{suh2022differentiable}
H.~J. Suh, M.~Simchowitz, K.~Zhang, and R.~Tedrake.
\newblock Do differentiable simulators give better policy gradients?
\newblock In {\em International Conference on Machine Learning}, pages
  20668--20696. PMLR, 2022.

\bibitem{sutton2018reinforcement}
R.~S. Sutton and A.~G. Barto.
\newblock {\em Reinforcement learning: An introduction}.
\newblock MIT press, 2018.

\bibitem{tripuraneni2021assessment}
N.~Tripuraneni, D.~Madeka, D.~Foster, D.~Perrault-Joncas, and M.~I. Jordan.
\newblock Assessment of treatment effect estimators for heavy-tailed data.
\newblock {\em arXiv preprint arXiv:2112.07602}, 2021.

\bibitem{verhoef2006out}
P.~C. Verhoef and L.~M. Sloot.
\newblock Out-of-stock: reactions, antecedents, management solutions, and a
  future perspective.
\newblock In {\em Retailing in the 21st Century}, pages 239--253. Springer,
  2006.

\bibitem{wen2019deep}
R.~Wen and K.~Torkkola.
\newblock Deep generative quantile-copula models for probabilistic forecasting.
\newblock {\em arXiv preprint arXiv:1907.10697}, 2019.

\bibitem{mqcnn}
R.~Wen, K.~Torkkola, B.~Narayanaswamy, and D.~Madeka.
\newblock A multi-horizon quantile recurrent forecaster.
\newblock {\em arXiv preprint arXiv:1711.11053}, 2017.

\bibitem{yang2022mqretnn}
S.~Yang, C.~Eisenach, and D.~Madeka.
\newblock Mqretnn: Multi-horizon time series forecasting with retrieval
  augmentation.
\newblock {\em arXiv preprint arXiv:2207.10517}, 2022.

\bibitem{Zeni2001censored}
R.~H. Zeni.
\newblock {\em Improved Forecast Accuracy in Airline Revenue Management by
  Unconstraining Demand Estimates from Censored Data}.
\newblock PhD thesis, Rutgers University, 2001.

\bibitem{zipkin2008old}
P.~Zipkin.
\newblock Old and new methods for lost-sales inventory systems.
\newblock {\em Operations Research}, 56(5):1256--1263, 2008.

\end{thebibliography}

\clearpage
\begin{APPENDICES}

	\section{Generalized Objective Function}\label{app:objective}

	\subsection{Objective Function}
	\label{sec:objective}
	Let $\fillnot^i_t$ denote the random variable for the fill-rate level for a product $i$ at time $t$. We adopt the methodology proposed by \cite{dada2007newsvendor} in that the fill-rate level is the number of units the vendor has available at the decision epoch. As a result, order quantities larger than the fill-rate level will be capped at the fill-rate level. The reward function at each decision epoch simply measures the cash flows, including the revenue from sales at each decision epoch, the cost of purchasing, and in-bounding inventory, following textbook approaches \cite{porteus2002foundations}:
	\begin{align}
		R^i_t & := \rewardrappendix \notag                                                   \\
		      & := \uunderbrace{p^i_t  \min (\demnot^i_t, \beforeinv)}{\text{Net Revenue}}
		- \longoverbrace{\costnot^i_t \min(\tilde{a}^i_t, \fillnot^i_t)}{\text{Cost of Goods Sold}}. \notag
	\end{align}
	We aim to maximize the total discounted reward, expressed as the following optimization problem
	\begin{align}
        \max_{\paramnot} & ~\rewardexp \biggl[\sum_{i \in \mathcal{A}} \sum^{\infty}_{t=0} \gamma^t  R^i_t(\paramnot)  \biggr] \label{eqn:objfun}                                                                    \\
		\text{subject to:}  \notag                                                                                                                                                                        \\
		                                           & I^i_0 = k^i \label{eqn:startinv}                                                                                                                     \\
		                                           & a^i_t = \policynot(H_t) \label{eqn:action}                                                                                                                \\
		                                           & \tilde{a}^i_t = \min(a^i_t, f^i_t) \label{eqn:rp_0}                                                                                                  \\
		                                           & I^i_{t+1} = \max \biggl[ I^i_t + \sum^{t}_{k=0} \tilde{a}^i_k  \mathbbm{1}_{\{\vltnot^i_k = (t-k)\}} - \demnot^i_t, 0\biggr] \label{eqn:evolution_0} 
	\end{align}
	subject to the following constraints.
	\begin{align}
		 & \sum_{i \in \mathcal{A}} \sum^{t}_{k=0} \tilde{a}^i_{t-k}  \mathbbm{1}_{\{\vltnot^i_k = (t-k)\}} \leq \text{Arrival Constraint}_t \label{eqn:arrivalconst} \\
		 & \sum_{i \in \mathcal{A}} I^i_t \text{Volume}^i_t \leq \text{Volume Constraint}_t \label{eqn:volconst}                                                      \\
		 & \sum_{i \in \mathcal{A}} \tilde{a}^i_t c^i_t \leq \text{Capital Constraint}_t \label{eqn:capconst}                                                         \\
		 & \sum_{i \in \mathcal{A}} \tilde{a}^i_t \leq \text{Unit Constraint}_t \label{eqn:unitconst}
	\end{align}
Our objective function (Equation \eqref{eqn:objfun}) and inventory
 balance equations (Equation \eqref{eqn:evolution_0}) come directly from the work of
 \cite{porteus2002foundations} and \cite{bellman1955optimal}, our fill
 rate (Equation \eqref{eqn:rp_0}) mimics
 \cite{dada2007newsvendor}. Our constraints are a standard
 formulation, with Equation \eqref{eqn:arrivalconst} being drawn from
 the work of \cite{kaplan1970dynamic}. Equations \eqref{eqn:volconst},
 \eqref{eqn:capconst}, and \eqref{eqn:unitconst} are standard volume,
 unit, and capital constraints.  In this full formulation, there are
 multiple constraints, equations
 \eqref{eqn:arrivalconst}-\eqref{eqn:unitconst}, that are expressed over the
 full set of products at each decision epoch $t$.

	\clearpage
	\section{Featurization}\label{app:featurization}

	The state at time $t$ for
	product $i$ now contains a superset of the features as used in
	\cite{mqcnn} and listed here:
	\begin{enumerate}
		\item The current inventory level $\invnot^i_{(t-1)}$
		\item Previous actions $a^i_u$ that have been taken $\forall u < t$
		\item Demand time series features
		      \begin{enumerate}
			      \item Historical availability corrected demand
			      \item Distance to public holidays
			      \item Historical website glance views data
		      \end{enumerate}
		\item Static product features
		      \begin{enumerate}
			      \item Product group
			      \item Text-based features from the product description
			      \item Brand
		      \end{enumerate}
		\item Economics of the product - (price, cost etc.)
	\end{enumerate}

	The state also contains lags of these variables. Finally, in addition to the total current on hand inventory, inflight inventory is part of the state as well.

	Denote by $\mu(x_1, \ldots, x_m)$ the empirical mean and $\sigma(x_1,\ldots,x_m)$ the empirical standard deviation of a collection of variables $(x_1,\ldots,x_m)$. Let $T$ denote the total number of time periods in the training portion of the data, and $N$ denote the total number of products that we are considering. We consider the following feature transformations:

	\begin{align}
		\text{Group Center}(x^i_t)     & := \frac{x^i_t - \mu(x^i_0, \ldots, x^i_T)}{\sigma(x^i_0,\ldots,x^i_T)} \label{eqn:grouptransform}                                           \\
		\text{Center}(x^i_t)           & := \frac{x^i_t - \mu(x^0_0, \ldots, x^0_T, \ldots, x^N_0, \ldots, x^N_T)}{\sigma(x^0_0, \ldots, x^0_T, \ldots, x^N_0, \ldots, x^N_T)} \label{eqn:center} \\
		\text{Box Cox}(x^i_t, \lambda) & := \frac{(x^i_t)^{\lambda}-1}{\lambda} \label{eqn:boxcox}
	\end{align}

	Equation \eqref{eqn:grouptransform} takes the mean and standard deviation of the value of a feature for a particular product across time and uses those to standardize the value of that feature for that product. Equation \eqref{eqn:center} takes the mean and standard deviation of a feature's value across products and across time and then uses those to standardize the value of that feature for that product. We only use values from the ``in-sample'' training period, as defined in Section \ref{sec:data}.

	We perform a number of preprocessing steps for each feature. For the Asynchronous Actor Critic algorithms, we use Equation \eqref{eqn:boxcox} on the price ($\lambda=1.068$), cost ($\lambda=0.013$) and demand ($\lambda=0.317$) features to make the input feature distributions close to normal. For the other feature inputs to the A3C algorithms, we apply a ``Group Center'' as defined in Equation \eqref{eqn:grouptransform}. For the Augmented Random Search and Soft-Actor Critic algorithms, we apply Equation \eqref{eqn:grouptransform} to all features.

	For ARS and SAC, features encoding calendar events and holidays  receive a distinct treatment. The representation is  a time-series for each holiday $h$
	\begin{equation}
		f^h_t = c_1 e^{-{c_2}/{d^h_t}},
	\end{equation}
	where $d^h_t$ is the distance in units of time to the closest holiday $h$ such that the distance is positive when $t<t_h$, and negative when $t_h<t$. $d^h_t$ thus flips its smallest negative value to the largest positive value at a halfway point between the two occurrences of the same holiday $h$. The exponential transform inverts this. Constants $c_1$ and $c_2$ are chosen so that $f^h_t \approx 0$ at this halfway point and $f^h_t = 1$ when $t=t_h$. The two constants flip signs when $t>t_h$ so that $f^h_{t_h+1} = -1$ and approaches zero  from below when $t$ approaches the following midpoint.

	For DirectBackprop, we apply {\bf no} transformations to the raw features as they were not necessary to obtain convergence. Our baseline algorithms also required no feature preprocessing.

	\clearpage
	\section{Policy Network Structures}\label{app:network}

	\begin{figure}[htb!]
		\centering
		\includegraphics[width=0.7\linewidth]{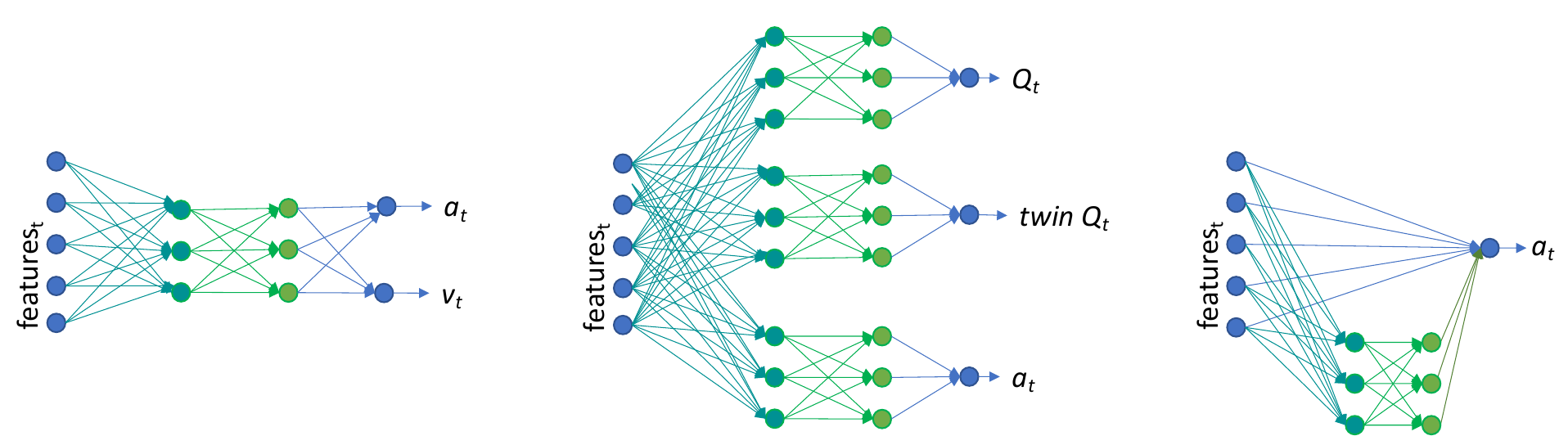}
		\caption{Policy network structures used with A3C, SAC and ARS respectively.}
		\label{fig:network_structures}
	\end{figure}

	Figure \ref{fig:network_structures} depicts the network structures used (from left to right) with the A3C, SAC and ARS algorithms respectively. The network structure used with A3C is an MLP with hidden sizes (128,128). We apply batch-normalization to each hidden layer. We use hyperbolic tangent activation functions at each hidden layer.

	The network structure used for SAC has a separate MLP for the value function and policy. Twin-Q is used to prevent overestimation. The hidden layer sizes for both the Q-function and the policy function are (256, 256). We use hyperbolic tangent activation functions at each hidden layer. The actions are taken from the policy network with a spherical Gaussian noise added.

	The network structure used with ARS has a skip connection and hidden layer sizes that were (500,50).

	\subsection{DirectBackprop}
	The DirectBackprop policy network consists of a Wavenet encoder with a MLP decoder \cite{mqcnn}. The encoder contains a stack of 5 dilated convolutions with rates [1, 2, 4, 8, 16]. Each convolution layer has 20 filters. The decoder MLP has two hidden layers of size 32 with ELU activation functions.

	We also ran experiments using the same network architecture as ARS and obtained similar results (although in that case, we applied feature transforms like those used for ARS).

	\clearpage
	\section{Algorithm Details}
	\label{sec:alg-details}

	\begin{algorithm}
		\caption{Direct Backpropagation for Inventory Control}
		\label{alg:dbp}
		\begin{algorithmic}
			\State \textbf{Input: } $\cD$ (a set of products), batch size $M$, initial policy parameters $\btheta_0$
			\State $b \gets 1$
			\While{not converged}
			\State Sample mini-batch of products $\cD_M$
			\State $J_b \gets 0$
			\For{$i \in \cD_M$}
			\State $J_{i}\gets 0$, $I_0^i \gets k^i$
			\For{$t = 0,\dots,T$}
            \State $a^i_{t} = \barepolicynot^i_{\paramnot_b, t}(H_{t})$
            \State $I^i_{t+1} \gets \mathcal{T}_t(H_t, \paramnot_b)$
            \State $J_{i}\gets J_{i} + \gamma^t R(H_t, \paramnot_b)$
			\EndFor
			\State $J_b \gets J_b + J_{i}$
			\EndFor
			\State $\btheta_{b} \gets \btheta_{b-1} + \alpha \nabla_{\btheta} J_b$
			\State $b \gets b+1$
			\EndWhile
		\end{algorithmic}
	\end{algorithm}

	Note that, for this algorithm - at time $t$ the policy only uses the values of the exogenous variables upto time $t$.

	\clearpage
	\section{Software Design}\label{app:software}
	\textbf{Environment}: The simulator built upon historical data uses the OpenAI Gym interface \cite{brockman2016openai}. Given a set of products, each episode corresponds to a single item's past data. An RL agent interacts with all items repeatedly during training. We register the gym environment as a custom environment.

	\textbf{Training}: Due to the large amount of data each instance of the environment has to store, we use the Plasma object store from Ray \cite{moritz2018ray} to efficiently transfer data across all processes. The objects are held in a shared memory so that they can be accessed by other RL workers without creating in-memory copies. Figure \ref{fig:infra} is a visual depiction of how a typical actor-critic agent works in our setting. Each rollout worker accesses the historical data from the plasma store, performs a rollout and then asynchronously returns values to the policy network.

	\begin{figure}[htb]
		\centering
		\begin{minipage}{0.45\textwidth}
			\centering
			\includegraphics[width=\linewidth]{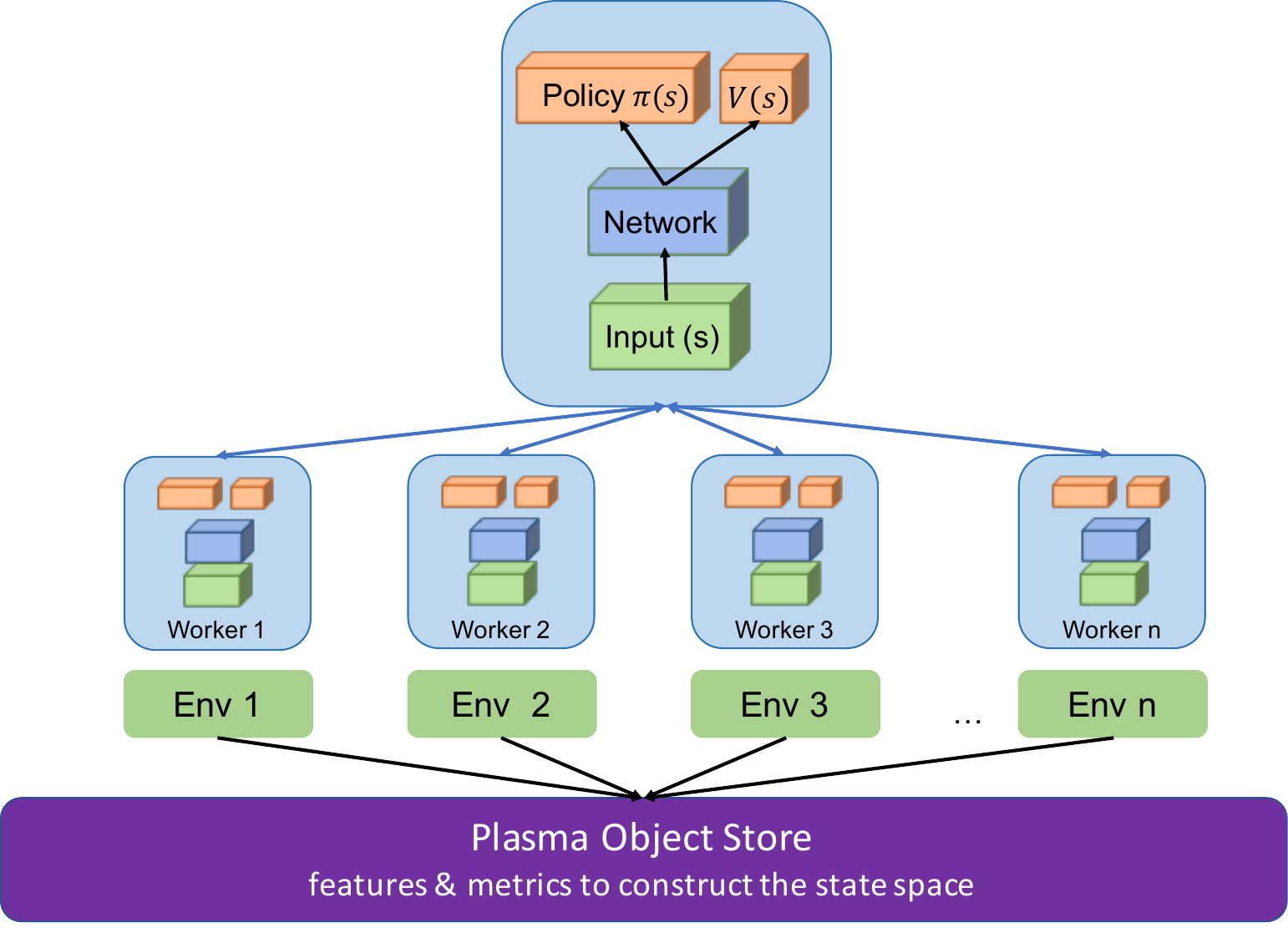}
			\caption{Visual depiction of the training architecture.}
			\label{fig:infra}
		\end{minipage} \hfill
		\begin{minipage}{0.5\textwidth}
			\centering
			\includegraphics[width=\linewidth]{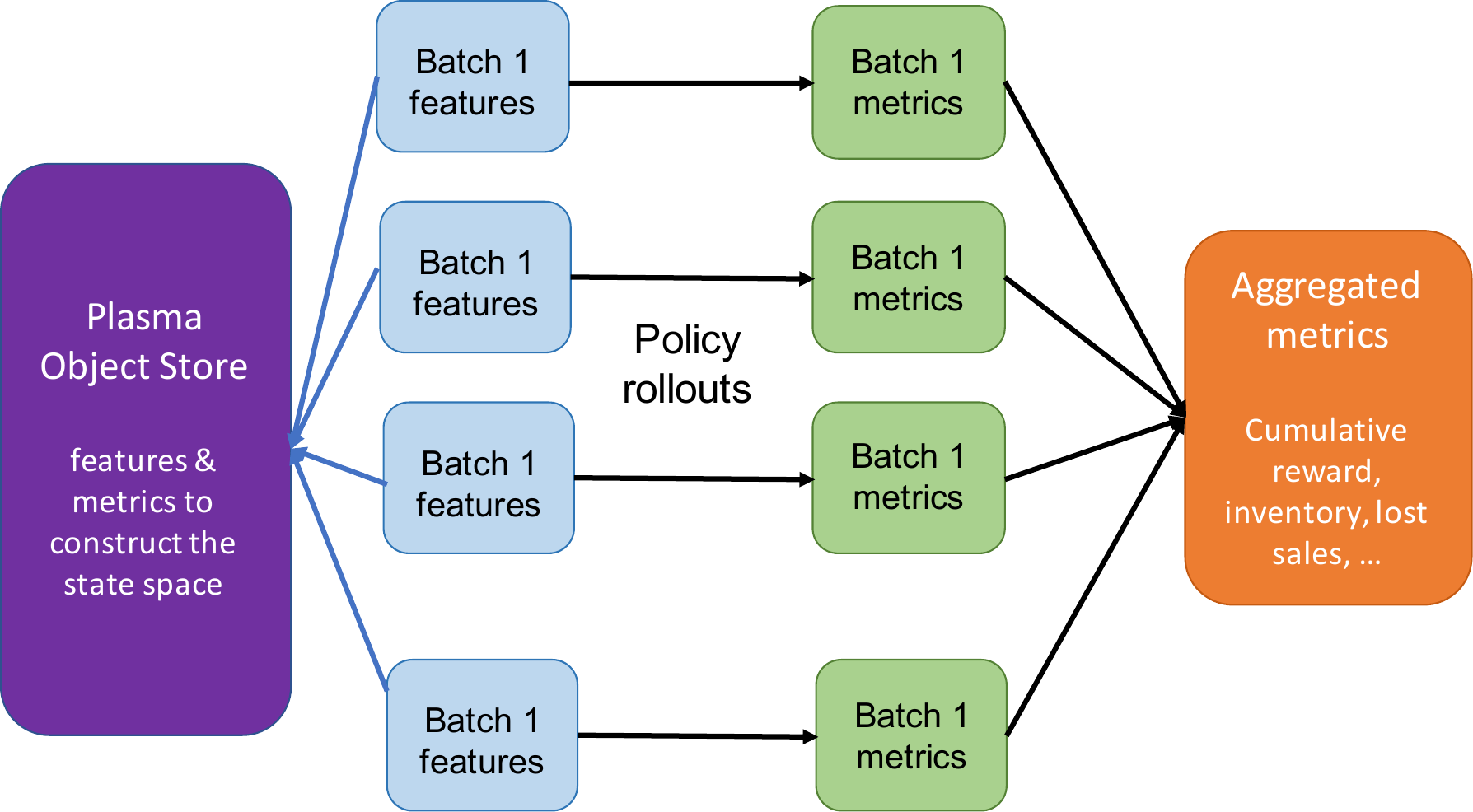}
			\caption{Visual depiction of the backtesting architecture.}
			\label{fig:backtesting_infra}
		\end{minipage}
	\end{figure}

	\textbf{Backtesting}: Unlike most RL applications (e.g., video games) where researchers mostly sample some observations for evaluation, backtesting needs to be performed for every single product in our case as it enables us to track both the overall profit and each individual product's behavior. Sequential evaluation thus becomes a bottleneck when the number of products scales up. To mitigate this issue, we implemented distributed backtesting using the remote actor interface of Ray. As shown in Figure \ref{fig:backtesting_infra}, each actor is responsible for backtesting a batch of products, and the results are later concatenated for aggregated metrics.

  \clearpage
  \section{Proofs}

\label{app:proofs}

  In this section, we recreate a component of the Simulation Lemma from \cite{abbeel2005exploration}.

  \begin{proposition}
    \label{proposition:simulationlemma}
  Let $\mathbb{P}, \hat{\mathbb{P}}$ be two probability distributions over a bounded domain $\mathcal{X}$, let $f$ be a bounded random variable over $\mathcal{X}$. Then:
  \begin{displaymath}
    |\mathbb{E}^{\mathbb{P}}[f] - \mathbb{E}^{ \hat{\mathbb{P}}}[f]| \leq (\sup_{x \in \mathcal{X}} f(x) - \inf_{x \in \mathcal{X}} f(x)) \tv{\mathbb{P}, \hat{\mathbb{P}}}
  \end{displaymath}
  \end{proposition}

  \begin{proof}~
{\bf Proof:}
    The proof is straightforward.  Define
\begin{displaymath}
g := \frac{f - \inf_{x \in {\cal
X}} f(x)}{\sup_{x \in {\cal X}} f(x) - \inf_{x \in {\cal X}} f(x)}
\end{displaymath}

For the case where $\sup_{x \in {\cal X}} f(x) - \inf_{x \in {\cal X }} f(x) = 0$, the proof follows immediately.

When $\sup_{x \in {\cal X}} f(x) - \inf_{x \in {\cal X }} f(x) > 0$, since $0 \leq g(x) \le 1$ we have:
\begin{displaymath}
\frac{    \left|\mathbb{E}^{\mathbb{P}}[f] - \mathbb{E}^{\hat{\mathbb{P}}}[f]\right| }{ \sup_{x \in {\cal X}} f(x) - \inf_{x \in {\cal X} f(x)} }
    =     \left|\mathbb{E}^{\mathbb{P}}[g] -
\mathbb{E}^{\hat{\mathbb{P}}}[g] \right|
\le \sup_{h: 0 \leq h \le 1}  \left|\mathbb{E}^{\mathbb{P}}[h] -
\mathbb{E}^{\hat{\mathbb{P}}}[h] \right|\equiv \tv{\mathbb{P}, \hat{\mathbb{P}}}
\end{displaymath}
\end{proof}
  \hfill $\blacksquare$

\end{APPENDICES}

\end{document}